\RequirePackage{iftex}

\ifLuaTeX

\fi

\documentclass[runningheads]{llncs}

\usepackage{eccv}

\usepackage{eccvabbrv}
\usepackage{graphicx}
\usepackage{booktabs}
\usepackage[accsupp]{axessibility}

\usepackage[breaklinks,colorlinks,citecolor=eccvblue]{hyperref}

\usepackage{orcidlink}

\newtheorem{assumption}{Assumption}

\begin{document}

\title{A Mechanism-Driven Theory of Phase Transitions in Active Learning}

\author{Julia Machnio \orcidlink{0009-0007-2736-2967} \and
Mads Nielsen \orcidlink{0000-0003-1535-068X} \and
Mostafa Mehdipour Ghazi \orcidlink{0000-0002-8473-281X}}
\authorrunning{J.~Machnio et al.}

\institute{Pioneer Centre for AI, University of Copenhagen, Copenhagen, Denmark \\
\email{\{juma,madsn,ghazi\}@di.ku.dk}}

\maketitle

\begin{abstract}
Active learning (AL) performance is known to be budget-dependent, yet regimes are typically defined by heuristic label counts that fail to generalize across datasets or architectures. We characterize AL dynamics by reframing budget regimes as shifts in the dominant generalization mechanism. By reinterpreting PAC-style risk components as dynamic interacting terms, we prove that dominance shifts are structurally unavoidable, creating a moving bottleneck for generalization. We operationalize this using measurable proxies and a segmented regression procedure to identify a tripartite taxonomy: data-driven, transition, and model-driven phases. Our framework explains the long-standing observation that representativeness, coverage, and uncertainty strategies excel at different stages. Experiments across natural and medical imaging show that AL efficiency depends on the alignment between the strategy's inductive bias and the active bottleneck. Moreover, self-supervised representation shift transitions earlier along the labeling trajectory, highlighting the role of representation quality in shaping AL dynamics. Overall, this work provides a unified framework for the next generation of transition-aware AL algorithms. The code is available at: \url{https://github.com/juliamachnio/PALM}.


\end{abstract}

\section{Introduction}
\label{sec:intro}

\begin{figure*}[!t]
\centering
\includegraphics[width=0.93\textwidth]{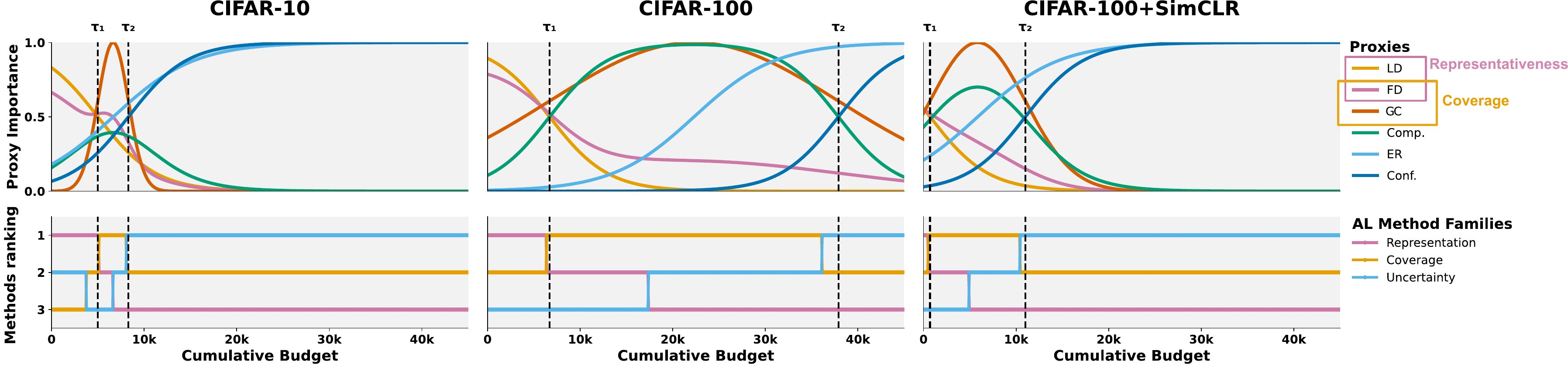}
\caption{\textbf{Active learning exhibits dynamic bottleneck transitions}. \textbf{Top:} Smoothed, normalized upper-bound proxies across labeling budgets. Distinct dominance regimes emerge: representativeness dominates early, coverage and complexity peak mid-stage, and empirical risk with statistical confidence dominate late. \textbf{Bottom:} Ranking of the methods. Transition points ($\tau$) denote detected shifts in proxy dominance. In low-complexity datasets like CIFAR-10, these transitions occur early, while in more complex dataset shifts them further along the trajectory. Pretrained models (e.g., SimCLR) accelerate these transitions, shortening the initial data-driven phase.}
\label{fig:summary}
\end{figure*}

Active learning (AL) aims to reduce annotation cost by iteratively selecting informative samples from a large unlabeled pool \cite{settles2009active}. Empirical studies consistently show that the effectiveness of different AL strategies depends strongly on the labeling budget \cite{hacohen2022active,yehuda2022active}. 
This behavior is commonly described through early-, mid-, and late-budget stages in which different families of methods dominate: representativeness-based strategies perform best early \cite{xu2003representative,aghaee2016active,voevodski2012active,hacohen2022active}, coverage-oriented approaches become effective at intermediate budgets \cite{sener2017active,yehuda2022active}, and uncertainty\nobreakdash-driven methods dominate later \cite{lewis1995sequential,gal2017deep,scheffer2001active,kirsch2019batchbald,ash2021gone}. However, these regimes are typically defined by heuristic label counts that vary across datasets, architectures, and representations, limiting their structural interpretability and provide limited guidance for the design of acquisition strategies. 

We hypothesize a different perspective: \textbf{AL regimes are not fixed} by budget, but emerge from shifts in the dominant generalization mechanism during sampling. Building on true risk decompositions under sampling bias and covariate shift \cite{valiant1984theory,menden2025bounds,cortes2008sample}, we reinterpret four structural components: empirical risk, distributional discrepancy, hypothesis-space complexity, and confidence term \cite{menden2025bounds}, as interacting terms whose relative influence evolves along labeling trajectory.

Within this perspective, it is important to distinguish between \emph{distributional representativeness} \cite{xu2003representative,voevodski2012active,hacohen2022active} and \emph{geometric coverage} \cite{sener2017active,yehuda2022active}. Representativeness measures how well the labeled subset reflects the underlying data distribution, controlling sampling-induced bias. In contrast, coverage quantifies the geometric dispersion of labeled samples in feature space and its impact on the hypothesis class. 
Both concepts appear in the AL literature \cite{sener2017active, hacohen2022active}, but their distinct effects on generalization dynamics have not been explicitly disentangled.

At different stages, different mechanisms may \emph{dominate} the learning process, \ie, yield the largest marginal effects on true risk minimization when adding new labels. As labeling progresses, the dominance of these mechanisms can change, creating a \emph{transition point}. Consequently, AL stages can be interpreted as structural phase transitions governed by the intrinsic interplay of dataset geometry, model capacity, and representation quality, rather than by arbitrary label counts.

To study this perspective empirically, we introduce measurable proxies associated with the components of the risk decomposition and track their evolution during the AL process. Using segmented regression along the labeling-budget axis, we identify a novel three-phase taxonomy: (1) a \emph{data-driven phase} dominated by distributional discrepancy, (2) a \emph{transition phase} governed by geometric coverage and complexity, and (3) a \emph{model-driven phase} ruled by empirical risk refinement. As revealed in \cref{fig:summary}, these transitions coincide with shifts in ranking of AL strategies and vary with dataset complexity and representation quality.

Experiments across multiple computer vision benchmarks and representation settings demonstrate that representativeness-, coverage-, and uncertainty-based strategies are effective only within specific labeling regimes. Aligning the acquisition strategy with the currently dominant mechanism yields larger reductions in the true risk bound, suggesting principled foundation for \emph{transition-aware AL}.

\textbf{Contributions}: 
(1) We introduce a mechanism-driven theory of AL that reframes heuristic budget regimes as shifts in dominant generalization mechanisms. 
(2) We disentangle distributional representativeness from geometric coverage, clarifying their distinct roles in adaptive sampling. 
(3) We propose operational proxies and a segmented regression framework to identify regime transitions along the labeling trajectory. 
(4) We empirically validate the resulting three-phase taxonomy across multiple datasets and representation settings.

\section{Related Work}
\label{sec:related}

\textbf{Uncertainty-based methods} prioritize low-confidence samples, assuming they are most informative for refining decision boundaries. Classical variants include margin \cite{scheffer2001active} and entropy  \cite{settles2009active}, while Bayesian methods such as BALD \cite{kirsch2019batchbald} and BAIT \cite{ash2021gone} leverage model disagreement or Bayes risk. Such strategies are typically most effective when the model is well-calibrated, most often in later budgets.

\textbf{Representativeness-based methods} aim to capture the structure of the underlying data distribution by selecting diverse, typical samples. Clustering or medoid approaches \cite{xu2003representative,aghaee2016active} select central samples, while TypiClust \cite{hacohen2022active} uses density-based typicality to identify representative samples. These methods reduce distributional discrepancy when model uncertainty is unreliable. Earlier theoretical work has also studied representation-based 
selection criteria \cite{wei2015submodularity, 
kaushal2019learning}.

\textbf{Coverage-based methods} promote broad exploration of the feature space using geometric criteria. Core-set selection \cite{sener2017active} minimizes the maximum distance to nearest labeled neighbors, while ProbCover \cite{yehuda2022active} extends this via probabilistic coverage. Unlike representativeness, coverage explicitly targets geometric dispersion in feature space and is often most effective at the low-to-mid budget.

\textbf{Hybrid approaches} combine multiple signals. BADGE \cite{ash2019deep} embeds gradients to mix uncertainty and diversity, UHerding \cite{bae2024uncertainty} combines uncertainty with coverage, and semantic approaches \cite{maalouf2022unified,liang2024semantic} further incorporate auxiliary cues. Their effectiveness depends on alignment with the current dominant mechanism.

Our work complements these efforts by grounding AL regime transitions in generalization-error decomposition and identifying shifts in dominant mechanisms as the cause, explaining why representativeness-, coverage-, and uncertainty-based strategies tend to perform best at different labeling budgets.

\section{Methodology}
\label{sec:methods}

We analyze AL dynamics through a decomposition of the generalization upper bound under sampling bias. Building on classical PAC-style analyses \cite{menden2025bounds,mohri2018foundations,valiant1984theory}, we reinterpret the bound components, empirical risk, distribution discrepancy, model complexity, and confidence, as \emph{dynamic mechanisms} whose relative influence evolves along the labeling trajectory. To study these shifts, we associate each component with measurable empirical proxies and track their evolution across labeling budgets to identify regime transitions.

\subsection{Preliminaries and Assumptions}
\label{subsec:prelims}

Let $\mathcal{Z}=\mathcal{X}\times\mathcal{Y}$ denote the input–label space with $(X,Y) \sim P$, where $X\in\mathcal{X}$ and $Y\in \mathcal{Y}$. Let $P_X$ denote the marginal distribution over inputs. 

Let $\mathcal{H}$ be a hypothesis class of predictors $h\,{:}\,\mathcal{X}\rightarrow\mathcal{Y}$ trained with loss $\ell(h(x),y)$. At acquisition step $t$, the active learner samples from an induced distribution $P_Q^{(t)}$ to construct a labeled set $S_t = \{(x_i, y_i)\}_{i=1}^{m_t}$ with $C$ classes, where $ m_t=|S_t|$ denotes the cumulative labeling  budget. Because sampling is guided by the acquisition strategy, the query distribution generally differs from the underlying data distribution, \ie $P_Q^{(t)}\neq P_X$, introducing sampling bias. Our analysis relies on the following assumptions:

\begin{enumerate}
    \item \label{aspt:loss} \textbf{Bounded Lipschitz loss.} The loss $\ell(h(x),y)\in[0,1]$ is $L$-Lipschitz in its first argument, enabling a concentration bound via Rademacher complexity~\cite{mohri2018foundations}.
    
    \item \label{aspt:capacity} \textbf{Finite capacity.} The hypothesis class $\mathcal{H}$ has finite empirical Rademacher complexity for any finite sample $S_t$, ensuring that empirical risk can converge to the true risk as the labeled set grows.
    
    \item \label{aspt:acquisition} \textbf{Non-degenerate acquisition.} For any measurable set $A \subseteq \mathcal{X}$ with $P_X(A) > 0$, there exists a finite time $t_A$ such that $P_Q^{(t)}(A) > 0$ for all $t \ge t_A$. 
    
    \item \label{aspt:growth} \textbf{Monotone growth.} The labeled set grows monotonically, $S_t \subset S_{t+1}$, with $m_{t+1} > m_t$, consistent with standard pool-based AL protocols.
    
    \item \label{aspt:representation} \textbf{Fixed representation.} Feature embeddings $\phi(x)$ remain fixed within each AL episode.
\end{enumerate}

Assumptions \ref{aspt:loss} and \ref{aspt:capacity} are standard in learning theory, while \ref{aspt:acquisition} to \ref{aspt:representation} formalize the adaptive sampling and representation constraints considered in our analysis. Assumption \ref{aspt:acquisition} ensures the per-step bound remains non-vacuous as the trajectory evolves: by guaranteeing eventual coverage of any positive-measure region, it prevents the query distribution from collapsing in a way that would render the discrepancy term uninformative. Assumption \ref{aspt:representation} holds directly for pool-based AL with frozen SSL backbones; in the supervised setting, it is applied analogously to a converged encoder within each AL episode.

\subsection{Generalization Bound Under Sampling Bias}
\label{subsec:bound}

We analyze AL dynamics through a decomposition of the true risk under sample selection bias \cite{cortes2008sample, mohri2018foundations, menden2025bounds}. \cref{thm:biased_bound} restates the bound from \cite{menden2025bounds} using our notation and provides the theoretical foundation for our phase analysis.

\begin{theorem}
\label{thm:biased_bound}
Let $\ell$ be bounded in $[0,1]$ and $L$-Lipschitz in its first argument. With probability at least $1-\delta$, for any $h \in \mathcal{H}$ trained on a labeled set $S$ of size $m$ drawn from an induced marginal $P_Q$, the true risk satisfies:
\begin{equation}
\begin{split}
\underbrace{R(h)}_{\text{True Risk}} & \leq 
\underbrace{\hat{R}_S(h)}_{\text{Empirical Risk}} + 
\underbrace{d_{\mathcal{F}}(P_X, P_Q)}_{\substack{\text{Distribution} \\ \text{Discrepancy}}} 
 + \underbrace{2 \, \text{Rad}(\ell \circ \mathcal{H} \circ S)}_{\substack{\text{Model} \\ \text{Complexity}}} +
\underbrace{\alpha \sqrt{ \frac{2 \log(4/\delta)}{m} }}_{\substack{\text{Confidence} \\ \text{Term}}} \,,
\end{split}
\label{eq:main_bound}
\end{equation}
where $d_{\mathcal{F}}$ is an integral probability metric (IPM) over a function class $\mathcal{F}$ of bounded complexity \cite{sriperumbudur2012empirical}, and $\alpha>0$ depends on the Lipschitz constant $L$.
\end{theorem}

\cref{eq:main_bound} decomposes true risk into four interacting mechanisms. Although \cref{thm:biased_bound} holds for any fixed query distribution satisfying Assumptions \ref{aspt:loss}--\ref{aspt:capacity}, each acquisition step $t$ induces a well-defined marginal $P_Q^{(t)}$ in the AL setting. Applying the bound per-step, treats each $P_Q^{(t)}$ as the relevant fixed distribution for that step and enables trajectory-level dominance analysis without requiring a single global $Q$. Unlike classical \iid learning, where bound components typically decrease with $m$, this per-step adaptive sampling induces competing and potentially non-monotonic effects:
\begin{itemize}
    \item \textbf{Empirical risk} may increase if newly queried samples lie in difficult or high-loss regions near the decision boundary.
    \item \textbf{Discrepancy} may increase if the acquisition rule oversamples atypical regions of $\mathcal{X}$, shifting $P_Q$ further from $P_X$.
    \item \textbf{Complexity} may fluctuate depending on whether new samples expand the geometric support of $S$ or merely increase its local density.
\end{itemize}

\noindent While the confidence decays at $\mathcal{O}(m^{-1/2})$, the remaining terms can evolve non-monotonically (see Appendix), producing a ``moving bottleneck.''

\subsection{Dominance Shifts}
\label{subsec:dominance_shift}

The non-monotonic behavior of the components in \cref{eq:main_bound}, under multiple acquisition steps $t$, implies that their relative contributions may change over along the acquisition trajectory.

\begin{theorem}
\label{thm:dominance_shift}
Under the assumptions in \cref{subsec:prelims}, there exist AL trajectories $\{S_t\}_{t=1}^\infty$ for which the dominant bound component 
$B^\star(t) = \arg\max_k B_k(t)$, evaluated per-step under $P_Q^{(t)}$, 
changes at least  as $m_t {\to} \infty$. 
\end{theorem}
This structural instability motivates viewing AL trajectories as a sequence of functional phases, each associated with a different dominant mechanism governing the bound. It establishes that the dominance of components does not persist in time, motivating the validation of the existence of dominance order. A complete proof is provided in Appendix.

\subsection{Mechanism-Driven Dominance and Phase}
\label{subsec:phases}

To operationalize \cref{eq:main_bound}, we identify the dominant mechanism at each labeling budget. Let the four bound components be denoted $B_k(t)$, $k \in \{1,2,3,4\}$. For comparability, these components are standardized using their empirical mean and standard deviation across all budgets within an AL episode. 

\begin{definition}[Dominance]
At budget index $t$, the dominant mechanism is defined as $B_\star(t) = \arg\max_{k \in \{1,2,3,4\}} B_k(t).$
\end{definition}

\begin{definition}[Phase]
A phase is a maximal contiguous interval $T = [t_{\mathrm{start}},\, t_{\mathrm{end}}]$ such that $B_\star(t)$ remains constant for all $t \in T$. A phase transition occurs at budget $t+1$ whenever $B_\star(t+1) \neq B_\star(t)$.
\end{definition}
By \cref{thm:dominance_shift}, these dominance changes arise along AL trajectories under adaptive sampling, partitioning the learning process into mechanism-driven phases.

\subsection{Transition Detection}
\label{subsec:segmented_regression}

Empirical proxy trajectories are noisy due to stochastic batch acquisition and model training. To identify structural transition points, we apply segmented (piecewise-linear) regression along the labeling-budget axis. For each proxy $Z_k(t)$ corresponding to component $B_k(t)$, we fit:
\begin{equation}
Z_k(t) = \beta_{k,0} + \beta_{k,1} t + \sum_{j=1}^{J} \gamma_{k,j}\,(t - \tau_j)_+ + \epsilon_t,
\label{eq:segmented_reg}
\end{equation}
where $\beta_{k,0}$ and $\beta_{k,1}$ denote the intercept and initial slope. The hinge function $(t - \tau_j)_+ = \max(0, t - \tau_j)$ introduces a slope change $\gamma_{k,j}$ at breakpoint $\tau_j$, while $\epsilon_t$ captures residual noise. The number of breakpoints $J$ and their locations $\{\tau_j\}$ are selected using the Bayesian Information Criterion (BIC) \cite{hastie2009elements}. The union of all breakpoints detected across the four components defines candidate transition budgets.

\begin{definition}[Transition]
A budget $\tau$ is a global phase transition if: (1) at least two components exhibit a statistically significant slope change ($\gamma_{k,j} \neq 0$) in a neighborhood of $\tau$, and (2) the dominant component $B_\star(t)$ changes at~$\tau$.
\end{definition}

This data-driven procedure ensures that detected transitions reflect meaningful structural changes rather than run-specific noise. Anchoring these shifts to specific budgets enables empirical validation of our ``moving bottleneck'' theory.

\subsection{Operational Proxies for Bound Components}
\label{subsec:proxies}

Because the components of \cref{eq:main_bound} are not directly computable for deep AL networks, we introduce observable proxies $Z_k(t)$ that track their relative evolution along the AL trajectory. These proxies are designed as indicators of relative dominance rather 
than exact estimators of PAC-bound magnitudes. The proxy ordering reflects the ordering of the underlying bound components, not that absolute values match. Formally, we assume the proxies are order-preserving with respect to dominance shifts. Further theoretical justification for these choices is provided in Appendix.

\paragraph{Empirical Risk (ER).}

We quantify the marginal informativeness of new queries via the ER reduction on the most recently acquired batch $A_{t-1}$:
\begin{equation}
ER(t) = \hat{R}_{\mathrm{pre}}^{(t)} - \hat{R}_{\mathrm{post}}^{(t)},
\end{equation}
where $\hat{R}$ denotes the average cross-entropy (classification) loss on $A_{t-1}$ before and after retraining the model at episode $t$.

\paragraph{Distributional Discrepancy (LD, FD).}

We retain two complementary signals capturing the structure of $d_{\mathcal{F}}$. Label discrepancy (LD) measures representativeness through the total variation distance between class frequencies in the labeled set $S$ and the full dataset $X$:
\begin{equation}
\mathrm{LD}(S) = \frac{1}{2} \sum_{c=1}^{C} |p_{S}(c) - p_X(c)|.
\end{equation}
While the feature discrepancy (FD) measures alignment in the representation space:
\begin{equation}
\mathrm{FD}(S) = \frac{1}{|D|} \sum_{x \in D} \min_{z \in S} \left(1 - \cos(\phi(x), \phi(z))\right),
\end{equation}
where $\phi(\cdot)$ are $\ell_2$-normalized embeddings and $D$ is a reference set. Notably, FD is inherently entangled, reflecting both distributional representativeness and geometric dispersion.

\paragraph{Geometric Coverage (GC).}

To isolate geometric coverage from representativeness, we measure the mean nearest-neighbor distance within the labeled set:
\begin{equation}
\mathrm{GC}(S) = \frac{1}{|S|} \sum_{z \in S} \min_{z' \in S \setminus \{z\}} \left(1 - \cos(\phi(z), \phi(z'))\right).
\end{equation}
GC captures manifold exploration and redundancy that are structurally distinct from LD and FD (see Appendix).

\paragraph{Model Complexity (Comp.).}

We approximate the empirical Rademacher complexity via the $\ell_2$ norm of the trained model parameters: $\mathrm{Comp}(t) = \sum_w \|w\|_2$. This proxy reflects how the geometry of the labeled set constrains or expands the effective hypothesis space.

\paragraph{Confidence Term (Conf.).}

We follow the deterministic rate in \cref{eq:main_bound} using the labeled set size $m_t$: $\mathrm{Conf}(t) = \alpha\, m_t^{-1/2}$ served as a vanishing baseline mechanism.

\subsection{Functional Alignment}
\label{subsec:alignment}

The structural instability of the generalization bound in \cref{thm:dominance_shift} implies that no single acquisition strategy is globally optimal. We therefore hypothesis that AL efficiency depends on the functional alignment between the method's selection bias and the currently dominant bound component.

\begin{proposition}
Let $B_\star(t)$ denote the dominant mechanism at budget $t$. An acquisition strategy $\mathcal{A}$ that minimizes the proxy $Z_\star$ associated with $B_\star$ is expected to yield larger marginal reduction in true risk than strategies optimized for non-dominant components.
\end{proposition}

This perspective suggests that the apparent failure of certain AL methods at specific budgets is not an inherent flaw of the algorithm, but a mismatch between the method's objective and the current bottleneck in the generalization bound. Specifically, our framework predicts a tripartite alignment:
\begin{enumerate}
    \item \textbf{Representativeness-based} methods (\eg, TypiClust \cite{hacohen2022active}) excel when \emph{distributional discrepancy} dominates, typically at early budgets when the labeled set poorly reflects $P_X$.
    \item \textbf{Coverage-based} methods (\eg, $k$-Center \cite{sener2017active}) perform best when geometric support expansion governs learning and \emph{hypothesis-space complexity} becomes the limiting factor.
    \item \textbf{Uncertainty-based} methods (\eg, Entropy \cite{settles2009active}) dominate during the \textit{empirical refinement} phase, where informative boundary samples are required to reduce residual training loss.
\end{enumerate}
By reframing heuristic budget regimes as a sequence of phase-specific optimizations, our framework provides a principled basis for strategy switching and motivates the design of transition-aware AL algorithms.

\section{Results and Discussion}
\label{sec:experiments}

\subsection{Experimental Setup}
\label{sec:settings}

\textbf{Datasets and Protocol.}
We evaluate our framework on natural and medical imaging datasets, including CIFAR-10/100 \cite{krizhevsky2009learning}, subsets of ImageNet \cite{he2016deep}, and ISIC \cite{gessert2020skin}. We follow a standard pool-based AL protocol with multi-scale labeling budgets $B$ to capture both low-budget behavior and high-budget convergence. ResNet-18/50 models \cite{he2016deep} are retrained for 100-200 epochs using SGD optimization. All implementation details and dataset specifications are provided in Appendix.

\textbf{Baselines.} 
To evaluate functional alignment, we select representative methods from each AL category: 
representativeness-based (TypiClust \cite{hacohen2022active}); coverage-based (Coreset \cite{sener2017active}, ProbCover \cite{yehuda2022active}); uncertainty-based (uncertainty \cite{lewis1995sequential}, entropy \cite{settles2009active}, and hybrid (BADGE \cite{ash2019deep}, UHerding \cite{bae2024uncertainty}). Random sampling is included as a reference baseline. To isolate the impact of feature quality from pretraining, we compare fully supervised training with self-supervised learning (SSL) backbones, including SimCLR \cite{chen2020simple}, BYOL \cite{grill2020bootstrap}, and MoCo v2+/v3 \cite{chen2020improved,chen2021empirical}. In the SSL setting, we train a linear classifier on frozen embeddings to decouple sampling geometry from representation drift. Additionally, we implement a switch among TypiClust, Coreset, and Uncertainty at detected transition points, named TCU.

\subsection{Proxy Dynamics and Generalization}
\label{subsec:proxy_dynamics}

We evaluate our framework by tracking the evolution of the proposed proxies alongside test accuracy. \cref{fig:proxy_evolution} illustrates these dynamics on CIFAR-100. Despite plotting raw proxy values, clear structural transitions emerge, including sharp decay of discrepancy and the non-monotonic fluctuations of ER and Comp., indicating that the identified regimes are intrinsic to the learning process rather than artifacts of downstream standardization. Additional results (CIFAR-10, CIFAR-100 with SimCLR, and ISIC) are provided in Appendix.

\begin{figure*}[!t]
\centering
\includegraphics[width=0.95\textwidth,height=0.18\textheight]{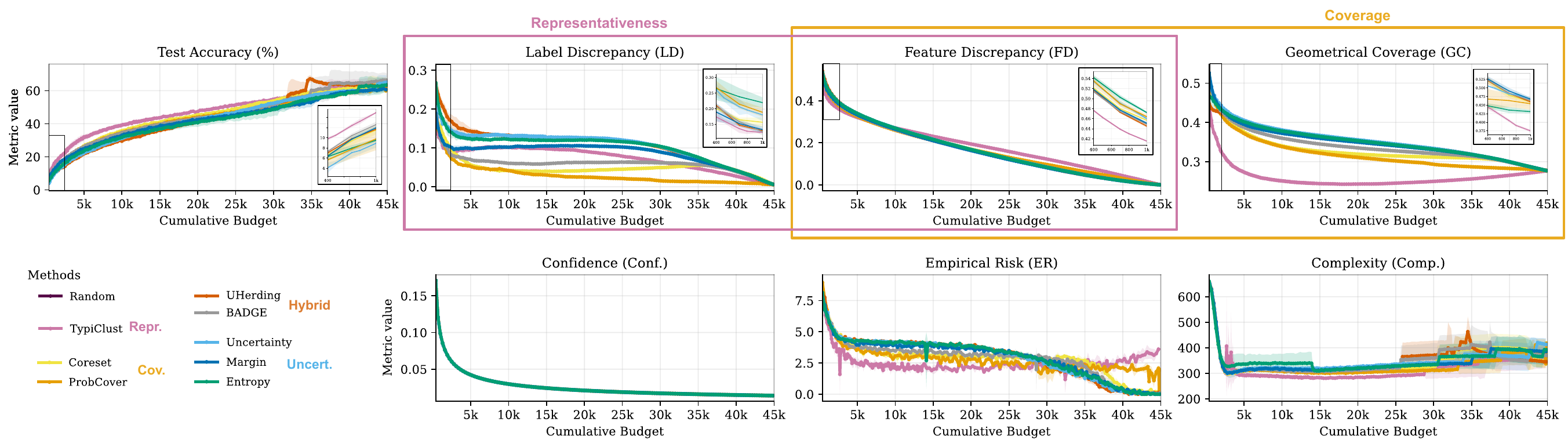}
\caption{Evolution of accuracy and operational proxies on CIFAR-100. Insets provide a granular view of cold-start dynamics. Representativeness proxies decay sharply in Phase I, where TypiClust yields superior marginal gains by minimizing distributional discrepancy. Geometric proxies exhibit discriminative stabilization during the Phase II transition, where coverage-based methods excel. ER and Comp. exhibit the non-monotonic behavior (\cref{thm:dominance_shift}), confirming that AL regimes are governed by shifts in dominant generalization mechanisms rather than fixed budget counts.}
\label{fig:proxy_evolution}
\end{figure*}

\textbf{Phase I: Data–Driven Regime.}
In the initial budget region ($\leq$5k labels), learning is dominated by \emph{distributional discrepancy}. LD and FD strongly differentiate methods, with TypiClust achieving the lowest values by prioritizing high-density ``typical'' samples. Although ER and Comp. decrease sharply at the beginning, this is consistent across methods and reflects the transition from an untrained state to a basic supervised signal rather than informative sampling.

During this stage, model calibration is insufficient for reliable uncertainty estimation. Consistent with our functional alignment principle (\cref{subsec:alignment}), minimizing discrepancy yields the largest performance gains, with TypiClust maintaining a $\sim$2\% accuracy advantage in this regime. Methods achieving the fastest reduction in LD and FD produce the strongest early accuracy gains.

\textbf{Phase II: Transition Regime.}
As the labeled set becomes representative ($\sim$5k–25k labels), the bottleneck shifts toward \emph{geometric coverage} and \emph{complexity}. GC and Comp.\ diverge across methods, while discrepancy proxies become less discriminative. Coverage-driven strategies (Coreset) begin to outperform TypiClust, maintaining lower GC values than uncertainty approaches while avoiding the local redundancy caused by TypiClust oversampling dense regions. The smoother decay of ER indicates a transition from learning global data structure to constraining the hypothesis space across the feature manifold, consistent with our predicted mid-budget dominance of coverage and complexity in \cref{subsec:alignment}.

\textbf{Phase III: Model–Driven Regime.}
Beyond $\sim$30k labels, distributional and geometric proxies converge across all methods, leaving \emph{empirical risk} refinement as the dominant mechanism. Uncertainty-based and hybrid methods (entropy, BADGE, and UHerding) achieve the largest reductions in ER and obtain the highest final accuracies. This late-stage behavior confirms the expected dominance shift toward the ER term in \cref{eq:main_bound}, providing a structural explanation for why uncertainty-based sampling becomes most effective only after the global distribution has been sufficiently characterized.

\textbf{Proxy Dynamics and Theoretical Alignment.}
The shift in dominant mechanisms is also reflected in the sampling behavior visualized in \cref{fig:tsne_img50}. TypiClust prioritizes the most ``typical'' points within clusters, rapidly reducing distributional discrepancy. In contrast, entropy sampling initially ignores global structure but later concentrates on ambiguous near-boundary regions (the ``voids'' between clusters), which become most informative for reducing ER.

Across all benchmarks, we observe the \emph{non-monotonic} behavior discussed in \cref{subsec:bound}. In particular, Comp.\ and ER fluctuate during the shift from exploitation to exploration. While these patterns are consistent across datasets, the \emph{temporal onset of each phase} varies and is identified quantitatively using segmented regression. These results support for the proposed ``moving bottleneck'' hypothesis.

\begin{figure*}[!t]
\centering
\includegraphics[width=0.88\textwidth,height=0.08\textheight]{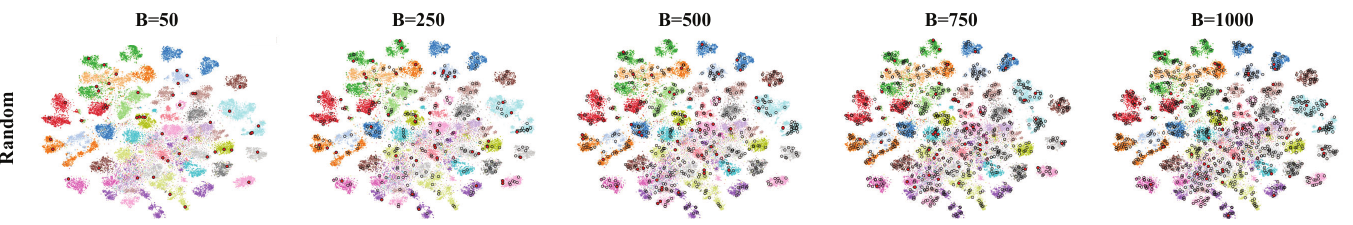}
\includegraphics[width=0.88\textwidth,height=0.08\textheight]{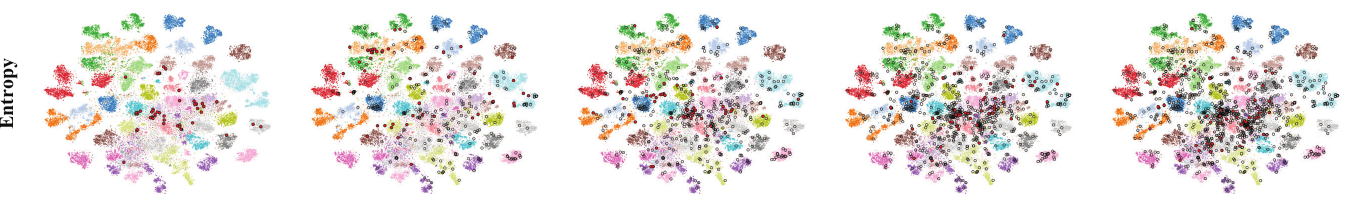}
\includegraphics[width=0.88\textwidth,height=0.08\textheight]{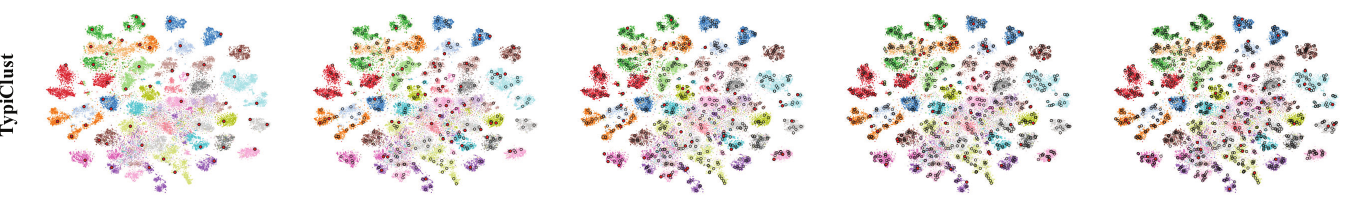}
\caption{Spatiotemporal selection dynamics on ImageNet-50. t-SNE visualizations (MoCov3 embeddings) compare AL methods across increasing budgets (B). Red points denote new acquisitions and black previously labeled. TypiClust prioritizes manifold typicality, reducing distributional discrepancy (Phase I), while Entropy targets ambiguous boundary regions to minimize risk. This shows how selection bias aligns with the dominant generalization bottleneck, explaining why TypiClust eventually suffers from local redundancy while Entropy remains effective for decision-boundary refinement.}
\label{fig:tsne_img50}
\end{figure*}

\subsection{Regime Identification}
\label{subsec:results_segmented_regression}

To qualitatively validate our proposed phase taxonomy, we perform segmented regression across four dataset-representation pairs. We evaluate regression models with $K\in\{1,\dots,5\}$ segments and compute the BIC, Sum of Squared Errors (SSE), and coefficient of determination ($R^2$) to identify structural breakpoints.

As reported in \cref{tab:pairwise_regime_table}, increasing $K$ naturally reduces SSE by capturing local fluctuations. However, $K=3$ provides the most parsimonious segmentation that still captures the major structural transitions in the AL trajectory. 
Across all benchmarks, this configuration yields consistently high $R^2$ values (>0.93), supporting the existence of three macro-regimes in the AL process.

\paragraph{Structural Variation and Data Complexity.} 

The temporal location of transition points $\tau$ strongly correlates with dataset complexity. 
Although CIFAR-10 and CIFAR-100 share the same total sample size, the second transition $\tau_2$ for CIFAR\nobreakdash-100 occurs approximately four times later than for CIFAR-10. This delay reflects the 10-fold increase in label space, which requires a substantially larger budget before the model can move from data-driven representation learning to model-driven refinement. For ISIC, the $R^2$ values are lower than for the CIFAR benchmarks, though still high. This observation aligns with our theoretical expectations, \ie, in highly imbalanced and noisy datasets, the ER stabilizes more slowly, preventing the clean phase transitions observed in balanced benchmarks.

\paragraph{Effect of Representation Quality.}

The influence of representation quality is most evident in the CIFAR-100 with SimCLR results. Because the SSL backbone provides a structured feature space, the model begins training with meaningful representation and partial confidence calibration. Hence, detected transition points shift earlier, \eg, $\tau_2$ moves from 37,900 labels in the supervised setting to 11,000 with SimCLR. At later budgets, proxy trajectories flatten, indicating that distributional alignment and geometric coverage are achieved with fewer samples. Therefore, the three phases are compressed toward early portion of labeling. This confirms that regime transitions depend not only on dataset structure but also on prior model knowledge. Even though the dataset remains identical, improved initialization shifts transition points, altering the method rankings, providing further empirical support for our proposed ``moving bottleneck'' hypothesis. 

\begin{table*}[t]
\centering
\caption{Piecewise regime analysis across dataset–representation pairs. Results quantify structural transitions in the AL trajectory by fitting $K$ segments regression models. While increasing $K$ reduces the SSE, $K=3$ provides the best balance between model parsimony (BIC) and macro-phase identification. $\tau$ represents absolute label counts at identified breakpoints. Note the significant shift in $\tau$ for CIFAR-100+SimCLR.}
\label{tab:pairwise_regime_table}
\setlength{\tabcolsep}{4pt}
\resizebox{\textwidth}{!}{%
\begin{tabular}{c cccc | cccc}
\toprule
& \multicolumn{4}{c|}{\textbf{CIFAR-10}} 
& \multicolumn{4}{c}{\textbf{ISIC}} \\
\cmidrule(lr){2-5}
\cmidrule(lr){6-9}
$K$ & BIC ↓ & SSE ↓ & $R^2$ ↑ & $\tau$
    & BIC ↓ & SSE ↓ & $R^2$ ↑ & $\tau$ \\
\midrule

1 
& -31182 & 1.82 & 0.960 & --
& -1382 & 0.05 & 0.828 & -- \\

2
& -42188 & 0.12 & 0.997 & 6700
& -1427 & 0.03 & 0.892 & 7000 \\

\textbf{3}
& \textbf{-43235} & \textbf{0.09} & \textbf{0.998} & \textbf{5000, 8300}
& \textbf{-1485} & \textbf{0.02} & \textbf{0.937} & \textbf{7000, 10000} \\

4
& -44008 & 0.07 & 0.998 & 500, 5500, 9100
& -1513 & 0.01 & 0.956 & 5000, 8000, 10000 \\

5
& -44680 & 0.06 & 0.999 & 500, 4200, 6700, 11800
& -1526 & 0.01 & 0.967 & 5000, 6000, 8000, 10000 \\

\midrule

& \multicolumn{4}{c|}{\textbf{CIFAR-100}} 
& \multicolumn{4}{c}{\textbf{CIFAR-100 + SimCLR}} \\
\midrule

1 
& -32638 & 1.19 & 0.985 & --
& -28528 & 2.44 & 0.815 & -- \\

2
& -33470 & 0.96 & 0.988 & 17700
& -31763 & 1.04 & 0.921 & 10600 \\

\textbf{3}
& \textbf{-34349} & \textbf{0.76} & \textbf{0.990} & \textbf{6700, 37900}
& \textbf{-33014} & \textbf{0.74} & \textbf{0.943} & \textbf{700, 11000} \\

4
& -34763 & 0.67 & 0.991 & 5900, 34600, 37900
& -34452 & 0.51 & 0.962 & 700, 7500, 16400 \\

5
& -35103 & 0.61 & 0.992 & 2700, 18300, 34600, 37900
& -35014 & 0.43 & 0.967 & 700, 6800, 13300, 22100 \\

\bottomrule
\end{tabular}
}
\end{table*}

\begin{figure*}[!t]
\centering
\includegraphics[width=0.91\textwidth,height=0.22\textheight]{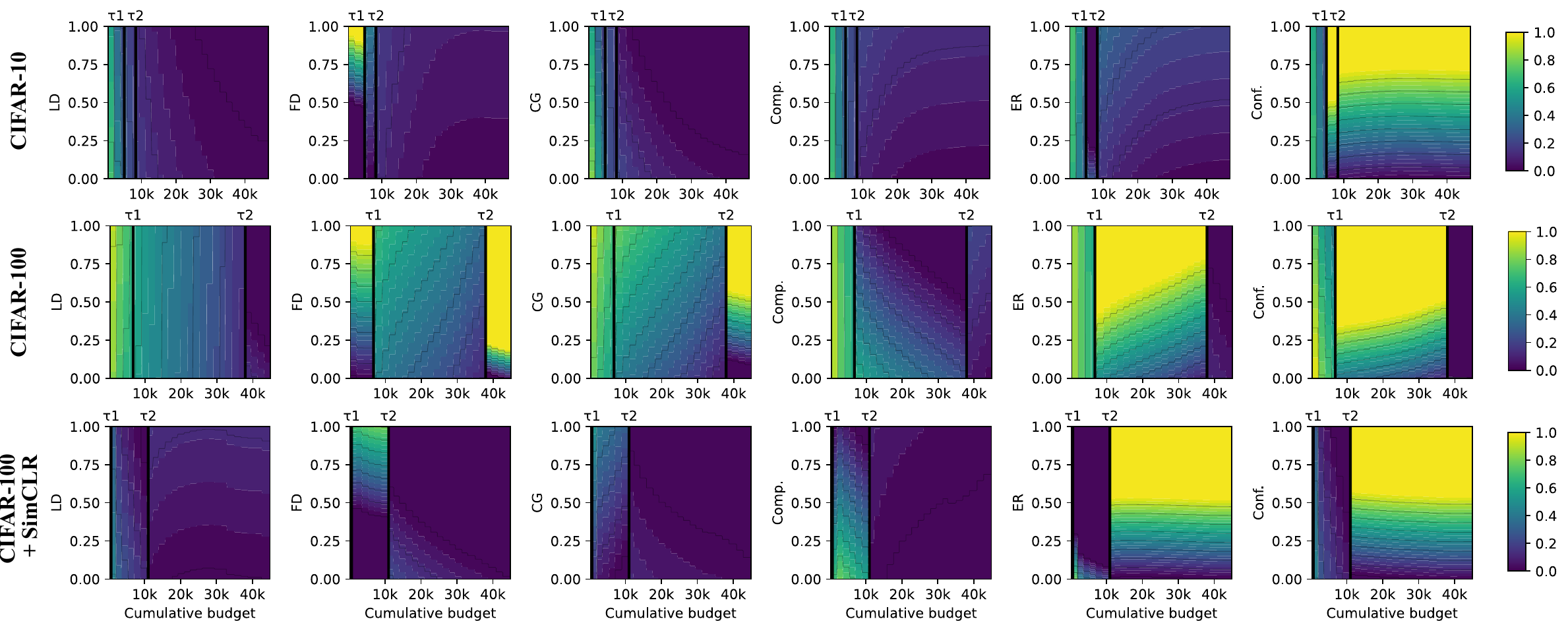}
\caption{Identification of AL regimes as probability maps of each proxy vs annotation budget. The $\tau$ points show detected phase transitions with piecewise regression.}
\label{fig:probability_maps}
\end{figure*}

\begin{figure*}[!t]
\centering
\includegraphics[width=0.97\textwidth,height=0.18\textheight]{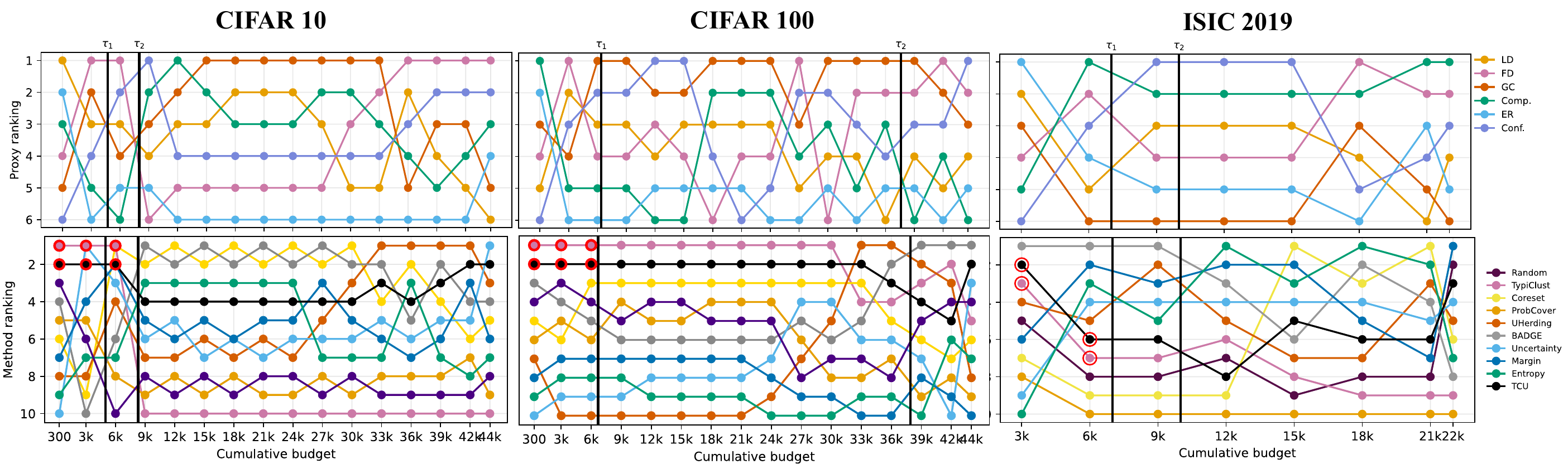}
\caption{Spatiotemporal alignment between proxy importance and acquisition strategy rankings. \textbf{Top:} Relative ranking of operational proxies based on their contribution to the generalization bound. \textbf{Bottom:} Corresponding accuracy-based rankings. Red circle refers to equivalent rank.}
\label{fig:ranking}
\end{figure*}

To visualize the evolution of the generalization bottleneck, \cref{fig:probability_maps} presents probability maps of each proxy as a function of the annotation budget. These maps quantify the relative influence of each generalization mechanism along the labeling trajectory. Consistent with our theoretical taxonomy, the early regime is dominated by LD, FD, and GC, indicating that representativeness and manifold support are the initial bottlenecks at low budgets. Once the trajectory passes $\tau_1$, the discriminative power of LD rapidly diminishes, but FD and GC remain significant for some time, reflecting the continued importance of geometric coverage. At the same time, ER and Conf.\ gradually increase in importance.

Beyond $\tau_2$, the trajectory enters the model-driven regime. For CIFAR-10, ER and Conf.\ become the dominant signals, indicating that performance improvements are driven mainly by ER refinement. In contrast, CIFAR-100 retains a noticeable secondary influence from Comp.\ and geometric proxies, suggesting that high-dimensional label spaces continue to benefit from additional manifold exploration even after basic coverage has been achieved.

The effect of representation quality is particularly evident for CIFAR-100 with SimCLR. While the first two phases resemble the supervised baseline, the final regime is dominated by ER. Because the SSL backbone provides a well-structured representation, the model rapidly achieves distributional alignment and geometric coverage. Consequently, the learning bottleneck shifts primarily to label-driven risk refinement, with minor influence from feature-space exploration.

\subsection{Alignment Between Proxy Dominance and Method Ranking}
\label{subsec:method_ranking}

To evaluate whether minimizing specific components of the theoretical bound improves downstream performance, we compare the ranking dynamics of proxies and AL methods across labeling budgets. As shown in \cref{fig:ranking}, no proxy or AL methods dominate the entire trajectory. Instead, the rankings evolve over time, reflecting shifts in the dominant generalization mechanism. While these plots highlight relative ordering, they do not convey absolute performance gaps.

\textbf{Representativeness Dominance.} 
Across all datasets, LD and FD consistently occupy the highest proxy rankings early on, with TypiClust asdominating method in this phase. For CIFAR-10 and ISIC, this advantage diminishes quickly, precisely when GC gains importance. In CIFAR-100, this representativeness-dominated phase persists longer, reflecting the larger label space and the greater budget required to achieve sufficient distributional coverage.

\textbf{Transition and Hybridization.} 
As the advantage of TypiClust declines, the ranking gradually shifts toward coverage-oriented and hybrid strategies. Coverage-based methods begin to rise right after the representativeness peak, while hybrid methods (UHerding), which combine multiple signals, gain prominence during the transition regime. In later stages, these hybrid strategies become comparable to pure uncertainty methods. This pattern is most clearly observed in CIFAR-10, where uncertainty and hybrid approaches dominate the final ranking.

\textbf{Complexity and Dataset Noise.} 
ISIC exhibits noisier ranking dynamics due to its class imbalance and label noise. Hybrid methods achieve high rankings earlier in the trajectory, particularly during transition, before uncertainty methods eventually dominate the model-driven phase. Despite this variability, the same qualitative relationship persists: proxies that rank highest at a given budget correspond to the methods achieving the best accuracy. Although each queried sample influences all proxies simultaneously, the most effective strategy is the one that strongly reduces the currently dominant mechanism.

\textbf{Structural Significance of Breakpoints.} The transition points $\tau$ identified via segmented regression coincide with the most volatile regions of ranking maps, where the largest number of rank shifts occurs among proxies and methods. This alignment indicates that the detected breakpoints reflects genuine structural shifts in learning dynamics rather than statistical artifacts. These points therefore mark transitions between regimes dominated by different generalization mechanisms.

\textbf{Post-Hoc Switching as Alignment Validation.}
To test whether the detected transition points $\tau$ have algorithmic meaning, we evaluate TCU, a post-hoc hard switch between TypiClust, Coreset, and Uncertainty, with boundaries set by segmented regression. As shown in \cref{fig:ranking}, TCU improves over static baselines by preserving TypiClust's early advantage and later exploiting coverage- and uncertainty-driven selection. Still, switching is limited when the chosen component is not competitive in its regime, suggesting that transition-aware AL should adaptively select among strong strategies rather than follow a fixed sequence.

\subsection{Cold-Start Dynamics and Representational Synergy}

\begin{figure}[t]
\centering
\includegraphics[width=0.81\columnwidth]{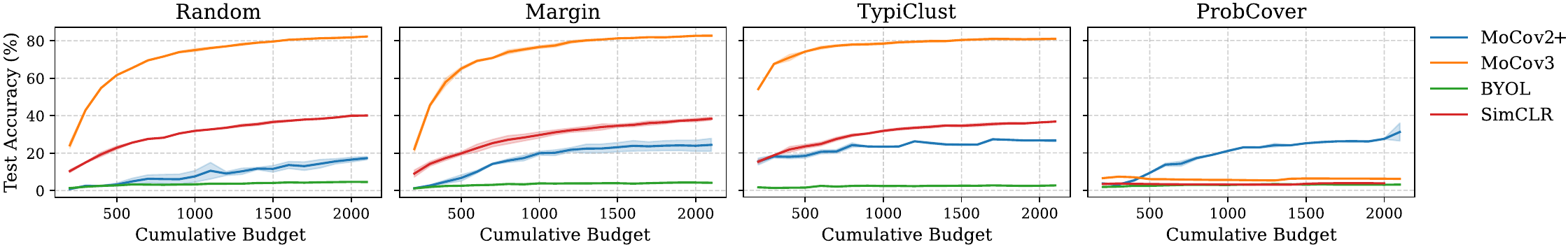}
\caption{Low-budget performance of the AL methods on ImageNet-100 using different embeddings. SSL enhances early-stage accuracy by providing meaningful features, particularly benefiting TypiClust. For ProbCover, the fixed-radius selection hampers training across embeddings.}
\label{fig:ssl_comparison_img100}
\end{figure}

As an additional analysis, we examine the cold-start regime bridging AL and few-shot or zero-shot scenarios. The earlier analyses presented in \crefrange{subsec:proxy_dynamics}{subsec:method_ranking} identified this region as a data-driven phase where representativeness forms the dominant bottleneck. However, methods that rely on it, such as TypiClust and ProbCover, remain highly dependent on the underlying representation quality. \cref{fig:ssl_comparison_img100} shows this sensitivity on ImageNet-100 using several SSL embeddings. 

TypiClust achieves strong early performance for some representations but not universally; using BYOL barely exceeds the random baseline, whereas with MoCov3 it surpasses 50\% accuracy at the same budget. These results highlight an important synergy between SSL and AL: sampling effectiveness depends strongly on how well the representation captures the underlying data geometry. 

Beyond accuracy improvements, representation quality also significantly affects computational efficiency (see Appendix). Leveraging pretrained SSL representations yields substantial training speedups, ranging from $30\times$ to over $160\times$ compared to training from scratch. This advantage is particularly relevant for out-of-distribution or domain-specific applications, where labeled data is scarce and rapid adaptation is essential.

\subsection{Efficiency and Scaling via Regime Awareness}

An important implication of our framework is that \emph{no single selection model remains optimal throughout the entire learning trajectory}. This observation motivates the use of regime-aware transition strategies that adapt the sampling mechanism as the dominant generalization bottleneck evolves. Importantly, this is not only an accuracy consideration but a computational scaling concern.

To quantify this effect, we evaluate the computational cost of AL rounds at small ($B=10$) and large ($B=1000$) budgets (see Appendix). The results shows that representativeness- and coverage-based methods (TypiClust, UHerding) scale substantially worse than uncertainty ones, often exceeding a runtime ratio of 2.0. For example, TypiClust reaches a ratio of 6.79 on CIFAR-10, making prolonged deployment on large datasets computationally impractical.

\section{Conclusion}
\label{sec:conclusion}

In this work, we moved beyond heuristic, budget-based descriptions of AL and provided a principled structural characterization of its dynamics. By reinterpreting PAC-style risk components as interacting mechanisms, we showed that the success of any AL strategy depends on its functional alignment with the currently dominant mechanism governing generalization. Our theoretical analysis further demonstrates that mechanism shifts are a structural consequence of adaptive sampling rather than an empirical artifact.

Empirically, we introduced a three-phase taxonomy, data-, transition-, and model-driven, that explains why representativeness-, coverage-, and uncertainty\nobreakdash-based methods excel at different stages of the labeling. Using measurable proxies for the components of the generalization bound, we showed that these phase transitions can be detected directly from the learning dynamics. Moreover, experiments with SSL reveal that transition points ($\tau$) are not fixed to absolute label counts but shift according the maturity of the underlying representation, effectively compressing early phases when strong prior knowledge is available.

These findings motivate the design of \emph{transition-aware AL algorithms}. Rather than relying on a single static acquisition rule, future systems could monitor proxy dynamics to automatically adapt sampling as the dominant generalization bottleneck evolves. Beyond improving accuracy, such transition-aware strategies could potentially offer computational benefits by enabling timely transitions away from expensive representativeness-based methods once their marginal utility diminishes. By providing this ``missing map'' of AL dynamics, this work unifies previously fragmented empirical observations and lays the foundation for more robust, automated, and theoretically grounded AL systems.

\section*{Acknowledgements}
This project is supported by the Pioneer Centre for AI, funded by the Danish National Research Foundation (grant number P1).

\bibliographystyle{splncs04.bst}
\bibliography{main.bib}

\newpage
\appendix
\chapter*{Appendix} \label{sec:appendix}

\section{Structural Properties of Bound Components}

\begin{lemma}[Non-monotonic empirical risk under adaptive sampling]
\label{lem:empirical_nonmonotone}
Let $\ell$ be a bounded loss and let $S_t \subset S_{t+1}$ with $S_{t+1}=S_t \cup A_t$. Suppose the learner returns an ERM $h_t \in \arg\min_{h\in\mathcal{H}} \hat R_{S_t}(h)$. Then there exist acquisition batches $A_t$ such that $\hat R_{S_{t+1}}(h_{t+1}) > \hat R_{S_t}(h_t)$.
\end{lemma}

\begin{proof}
Let $\mathcal{H}$ be a class with finite capacity. If $S_t$ consists of samples that are easily separable or lie far from the optimal decision boundary, $h_t$ attains low empirical loss. Let $A_t$ be an acquisition batch containing hard samples, those located within the margin of the current hypothesis or samples that are structurally inconsistent with the labels in $S_t$. Since $h_t$ is optimized only for $S_t$, its loss on the novel points in $A_t$ may be arbitrarily large. Because $\mathcal{H}$ must now minimize the joint loss over $S_t \cup A_t$, the introduction of these high-loss points forces a trade-off that can increase the average empirical risk $\hat{R}_{S_{t+1}}$. Thus, under adaptive sampling, empirical risk is not a monotonic function of $m_t$.
\end{proof}

\begin{lemma}[Vanishing discrepancy under non-degenerate acquisition]
\label{lem:discrepancy_vanishes}
Assume: (i) $\mathcal{F}$ is a class of bounded measurable functions; (ii) the acquisition process satisfies the non-degeneracy condition (Sec. 3.1); and (iii) $m_t \to \infty$. Then $P_Q^{(t)} \Rightarrow P_X$ and $d_{\mathcal{F}}(P_X, P_Q^{(t)}) \to 0$.
\end{lemma}

\begin{proof}
By non-degeneracy, every measurable set $A$ with $P_X(A)>0$ receives nonzero mass under $P_Q^{(t)}$ after some finite $t_A$. As $m_t \to \infty$ and $S_t$ grows monotonically, the empirical measure of $S_t$ converges uniformly over $\mathcal{F}$ by standard Glivenko--Cantelli arguments. Hence, $\sup_{f\in\mathcal{F}} | \mathbb{E}_{P_X}[f] - \mathbb{E}_{P_Q^{(t)}}[f] | \to 0$, which establish the convergence of $d_{\mathcal{F}}(P_X, P_Q^{(t)})$.
\end{proof}

\begin{lemma}[Non-monotonic empirical Rademacher complexity]
\label{lem:rademacher_nonmonotone}
Let $S_t \subset S_{t+1}$. There exist acquisition batches $A_t$ such that $\mathrm{Rad}(\mathcal{H} \circ S_{t+1}) > \mathrm{Rad}(\mathcal{H} \circ S_t)$.
\end{lemma}

\begin{proof}
Recall $\mathrm{Rad}(\mathcal{H}\circ S) = \mathbb{E}_\sigma [ \sup_{h\in\mathcal{H}} \frac{1}{m} \sum_{i=1}^m \sigma_i h(x_i) ]$. During the \textit{Transition Phase}, AL often shifts from exploitation (sampling near known support) to exploration (sampling from novel geometric regions). If $A_t$ includes samples from directions where $\mathcal{H}$ exhibits substantially larger variation, specifically through orthogonal expansion of the feature matrix rank, the supremum in the numerator (the ability to fit noise) increases faster than the $1/m$ normalization factor. Consequently, empirical complexity can temporarily increase as the geometric support of the labeled set expands into previously unsampled regions of the manifold.
\end{proof}

\subsection{Illustrative Linear Example: Complexity Expansion}

Consider linear predictors $\mathcal{H} = \{ x \mapsto w^\top x : \|w\|_2 \le B \}$ with $\|x\|_2 \le R$. Then the empirical Rademacher complexity is $\mathrm{Rad}(\mathcal{H}\circ S) = \frac{B}{m} \mathbb{E}_{\sigma} \| \sum_{i=1}^m \sigma_i x_i \|_2$.

\paragraph{Initial dataset.}

Let $S_t = \{x_i = \tfrac{1}{\sqrt{m}} e_1\}_{i=1}^m$. Then $\mathrm{Rad}_t = \Theta(\frac{B}{m})$.

\paragraph{Expanded dataset.}

Let $\{e_1, \dots, e_{m+1}\}$ be an \textbf{orthonormal basis} for the feature space. Let $S_t = \{x_i = \tfrac{1}{\sqrt{m}} e_1\}_{i=1}^m$ be the initial set, and let the acquisition batch be $A_t = \{e_2, \dots, e_{m+1}\}$ so that $S_{t+1} = S_t \cup A_t$ and $|S_{t+1}| = 2m$. The expectation of the norm satisfies $\mathbb{E} \| \sum_{i=1}^{2m} \sigma_i x_i \|_2 \ge \sqrt{1 + m}$. Therefore:
\[
\mathrm{Rad}_{t+1} = \frac{B}{2m} \sqrt{1+m} = \Theta\left(\frac{B}{\sqrt{m}}\right).
\]
Since $\frac{B}{\sqrt{m}} > \frac{B}{m}$ for $m > 1$, orthogonal exploration increases empirical complexity even as the total sample size doubles.

\subsection{Full Proof of Theorem 2}
\label{supp:formal_proof}

To prove that the identity of the dominant term must change, we consider the four components of the generalization bound from Eq. (1) as a time-varying vector $\mathbf{B}(t) = [B_1(t), B_2(t), B_3(t), B_4(t)]^\top \in \mathbb{R}_+^4$, where:
\begin{itemize}
    \item $B_1(t) = \hat{R}_{S_t}(h_t)$ \hfill (Empirical Risk)
    \item $B_2(t) = d_{\mathcal{F}}(P_X, P_Q^{(t)})$ \hfill (Distributional Discrepancy)
    \item $B_3(t) = 2\,\mathrm{Rad}(\ell \circ \mathcal{H} \circ S_t)$ \hfill (Model Complexity)
    \item $B_4(t) = \alpha \sqrt{2 \log(4/\delta)/m_t}$ \hfill (Statistical Confidence)
\end{itemize}
The active bottleneck at step $t$ is defined by $k^*(t) = \arg\max_{k \in \{1,2,3,4\}} B_k(t)$.

\paragraph{Step 1: Asymptotic decay of sampling and statistical terms.}

Under Assumption~4, the budget $m_t \to \infty$, hence $\lim_{t\to\infty} B_4(t) = 0$. By Lemma~\ref{lem:discrepancy_vanishes} and the non-degeneracy of the acquisition process (Assumption~3), we have $P_Q^{(t)} \Rightarrow P_X$. Since the Integral Probability Metric $d_{\mathcal{F}}$ metrizes weak convergence for bounded $\mathcal{F}$, it follows that $\lim_{t\to\infty} B_2(t) = 0$.

\paragraph{Step 2: Persistence of model-internal terms.}

By Lemma~\ref{lem:empirical_nonmonotone}, $\hat{R}_{S_{t}}$ is not necessarily monotonic. In the agnostic setting (non-zero Bayes error $\epsilon^* > 0$), the empirical risk remains bounded away from zero, $\liminf_{t\to\infty} B_1(t) \ge \epsilon^*$. Similarly, by Lemma~\ref{lem:rademacher_nonmonotone}, $B_3(t)$ can increase under geometric expansion and does not vanish as long as the hypothesis class $\mathcal{H}$ maintains non-zero capacity over the sampled support.

\paragraph{Step 3: Contradiction of fixed dominance.}

Assume for contradiction that the identity of the dominant term remains fixed as $k^*(t) = 2$ (Discrepancy) for all $t > T$. This implies $B_2(t) \ge B_k(t)$ for all $k \in \{1,3,4\}$. However, we have $\lim_{t\to\infty} B_2(t) = 0$, while $\liminf_{t\to\infty} B_1(t) \ge \epsilon^* > 0$. There must exist a finite time $t' > T$ such that $B_2(t') < \epsilon^* \le B_1(t')$, which implies $k^*(t') \neq 2$. This contradicts our assumption of fixed dominance.

\paragraph{Step 4: Existence of a structural shift.}

Since terms $B_2$ and $B_4$ vanish while $B_1$ and $B_3$ persist, any trajectory initialized in a regime dominated by data-driven bias ($B_2$) or statistical uncertainty ($B_4$) must eventually undergo a transition where $k^*(t)$ shifts to a model-driven component ($B_1$ or $B_3$). Thus, the active bottleneck $k^*(t)$ is non-constant, completing the proof. \hfill$\square$

\section{Proxy Justification}
\label{supp:proxy_justification}

In this section, we provide the formal rationale for mapping the abstract terms of the Generalization Bound (Eq.(1)) to the measurable proxies introduced in Sec. 3.6.

\subsection{Empirical Risk (ER) and Marginal Informativeness}

While the theoretical bound uses the absolute empirical risk $\hat{R}_S(h)$, in the context of an active learning trajectory, the \textit{change} in risk identifies the current bottleneck. Our use of marginal risk reduction ($ER = \hat{R}_{\mathrm{pre}} - \hat{R}_{\mathrm{post}}$) is inspired by the \textbf{Expected Error Reduction (EER)} framework. This proxy captures the "gradient" of refinement; a high $ER$ indicates that the acquisition step successfully queried samples that forced a significant update to the hypothesis, signaling that the model-driven refinement mechanism is active.

\subsection{Distributional Discrepancy (LD, FD)}

\paragraph{Label-space Discrepancy (LD).}

Using the Total Variation (TV) distance between class frequencies is the standard operationalization of representativeness in the presence of label shift. It directly approximates the $d_{\mathcal{F}}$ term when $\mathcal{F}$ is restricted to the space of class-indicator functions.

\paragraph{Feature-space Discrepancy (FD).}

Measuring the mean nearest-neighbor distance to a reference set is a common surrogate for the \textbf{Maximum Mean Discrepancy (MMD)} or \textbf{Wasserstein distance} in high-dimensional embeddings where direct density estimation is intractable. As noted in the main text, FD is an \textit{entangled} proxy because decreasing the distance to the reference set requires both matching the distribution support (representativeness) and filling the feature space (coverage).

\subsection{Geometric Coverage (GC) and Manifold Exploration}

Our definition of internal nearest-neighbor distance in the labeled set corresponds to the concept of the \textbf{$\delta$-cover} or \textbf{filling distance} in manifold learning. This is a direct proxy for the exploration or dispersion terms used in core-set-based active learning. By measuring internal distance, we isolate the geometric dispersion of the sample from its alignment with the target distribution.

\subsection{Complexity (Comp) and Norm-based Capacity}

Since computing the exact Rademacher complexity for deep neural networks is computationally intractable (NP-hard), we rely on \textbf{norm-based capacity control}. Extensive theoretical work has demonstrated that the generalization of overparameterized networks is governed by the magnitude of the weights (e.g., Frobenius or spectral norms) rather than the raw parameter count. Thus, the evolution of the $\ell_2$ norm of the parameters $\theta_t$ serves as a reliable indicator of how the geometry of the acquired data is constraining the hypothesis space.

\section{Proxy–Bound Dominance Consistency}
\label{sec:proxies}

To relate the observable proxy dominance to the structural dominance in the theoretical bound, we rely on the assumption that our proxies are order-preserving.

\begin{assumption}[Order-preserving alignment]
\label{ass:alignment}
For each component \\ $k \in \{1,2,3,4\}$, there exists a strictly monotone function $\psi_k$ such that $B_k(t) = \psi_k(Z_k(t)) + \varepsilon_{k,t}$, where $\varepsilon_{k,t}$ is a bounded perturbation satisfying $|\varepsilon_{k,t}| \le \eta_k$.
\end{assumption}

This assumption implies that as long as the signal-to-noise ratio of our proxies is sufficiently high, the observed phase transitions in the proxies correspond to transitions in the true generalization bottleneck.

\begin{proposition}[Proxy dominance implies bound dominance]
\label{prop:proxy_consistency}
Under Assumption~\ref{ass:alignment}, if for some $t$ and indices $i \neq j$, the proxy magnitudes are separated such that $|Z_i(t) - Z_j(t)| > \Delta$, with $\Delta > \frac{2(\eta_i + \eta_j)}{\min\{\psi_i', \psi_j'\}}$, then the dominance ordering between $Z_i, Z_j$ matches the ordering of the true bound components $B_i, B_j$.
\end{proposition}

\begin{proof}
Let $c = \min\{\psi_i',\psi_j'\} > 0$. By the Mean Value Theorem and triangle inequality:
\begin{equation}
|B_i(t) - B_j(t)| \ge c |Z_i(t) - Z_j(t)| - (\eta_i + \eta_j).
\end{equation}
If $|Z_i(t) - Z_j(t)| > 2(\eta_i + \eta_j)/c$, the right-hand side is strictly positive, preserving the sign of the difference and establishing dominance consistency.
\end{proof}

\section{Implementation Details}

\subsection{Data Specifications and Partitioning}
\label{supp:data}

\paragraph{Datasets Details.} 
\begin{itemize}
    \item \textbf{CIFAR-10/100}: Standard natural image benchmarks consisting of 50k training and 10k test images.
    \item \textbf{ISIC 2019}: Contains 25,331 clinical images of skin lesions. Due to the extreme class imbalance, we employ stratified sampling for the initial pool to ensure minority classes are represented early in the learning trajectory.
    \item \textbf{ImageNet Subsets}: We utilize subsets of 50, 100, and 200 classes sampled from the standard \textbf{ImageNet-1K} training set, following established benchmarking protocols to ensure a high-dimensional feature manifold for evaluating geometric coverage.
\end{itemize}

\paragraph{Preprocessing and Augmentation} 

During training, we apply: (i) Random resized cropping to the target resolution ($32\times32$ for CIFAR, $224\times224$ for ISIC/ImageNet); (ii) Random horizontal flipping ($p=0.5$); and (iii) Standard normalization using dataset-specific channel means and standard deviations.

For CIFAR experiments, we employ the standard CIFAR-style ResNet architecture, which replaces the original $7\times7$ stride-2 convolution and max-pooling layers with a $3\times3$ stride-1 convolution and removes early downsampling to preserve spatial resolution for $32\times32$ inputs.

\subsection{Self-Supervised Representation Details}
\label{supp:ssl_details}

For SSL-based experiments, we utilize fixed representations to decouple acquisition geometry from label-driven feature drift.

\paragraph{SimCLR Pretraining.} 

We pretrain ResNet-18 for CIFAR/ISIC and ResNet-50 for ImageNet for 500 epochs following. Training hyperparameters follow the SCAN protocol. Representations are extracted from the penultimate (GAP) layer, yielding 512-dimensional (ResNet-18) or 2048-dimensional (ResNet-50) embeddings.

\paragraph{Other SSL methods.}

For ImageNet, we also extract features using official weights from BYOL (ResNet-50 trained for 300 epochs with a batch size of 512), MoCov2+ (ResNet-50 trained for 800 epochs with a batch size of 256), and MoCov3 (ResNet-50 trained for 1000 epochs with a batch size of 4096) to verify that our observed phase transitions are not unique to a single SSL objective.

\paragraph{Frozen Backbones.} 

In the AL loop, encoder weights remain fixed. A single linear layer is appended to the embeddings and trained for 100 epochs at each AL step using the Adam optimizer ($lr=10^{-3}$). We deliberately choose this setting to decouple \textit{representation learning dynamics} from \textit{acquisition dynamics}. In a full fine-tuning regime, the feature manifold shifts at every episode, creating a ``moving target'' that obscures the structural transitions in the generalization bound. By fixing the encoder, we ensure that shifts in the active bottleneck, such as the transition from geometric coverage to uncertainty-based refinement, are a direct consequence of the acquired data geometry.

\subsection{Active Learning Protocol}
\label{supp:al_protocol}

\paragraph{Initialization and Cold-Start.} 

To strictly observe the behavior of acquisition strategies from the onset of the learning curve, we always start from empty labeled pool. For model-dependent strategies (e.g., Uncertainty-based methods) that require a trained hypothesis to compute acquisition scores, the first batch $b_0$ is selected via random sampling. This ensures all methods begin from an identical functional baseline before strategy-specific biases are introduced.

\paragraph{Acquisition Batch Sizes and Horizons.} 

At each iteration $t$, a batch of $b$ samples is selected from the unlabeled pool $\mathcal{U}_t$. We evaluate across diverse scales to ensure the phase transitions are not batch-size dependent. The specific configurations are:
\begin{itemize}
    \item \textbf{CIFAR-10}: Small-batch ($b=10, T=500$) and Large-batch ($b=100, T=450$).
    \item \textbf{CIFAR-100}: Standard-batch ($b=100, T=450$).
    \item \textbf{ISIC 2019}: Clinical-scale ($b=8, T=200$) and High-throughput ($b=100, T=22$).
    \item \textbf{ImageNet Subsets}: $b \in \{50, 100, 200\}$ class-aligned batches for $T=20$ iterations.
\end{itemize}

\paragraph{Repeated Experiments or Robustness Across Runs.}

Each experiment was performed using at least 3 independent random seeds for every AL method mentioned in Sec. 4.1. We report the mean performance and indicate the standard deviation (shaded regions in figures) to ensure the observed phase transitions are statistically robust.

\subsection{Segmented Regression Parameters}
\label{supp:regression_params}

The piecewise linear regression is implemented using the \texttt{pwlf} library. This procedure identifies the specific budget steps where the dominant learning mechanism shifts.

\begin{itemize}
    \item \textbf{Search Space}: For each dataset and proxy trajectory, we evaluate a range of $K \in \{1, \dots, 6\}$ linear segments. This range is sufficient to capture the tripartite phase structure (requiring at least 3 segments) while allowing for additional minor transitions.
    \item \textbf{Model Selection (BIC)}: To determine the optimal number of segments without over-fitting, we utilize the Bayesian Information Criterion:
    \begin{equation}
        \text{BIC} = n \ln(\text{SSE}/n) + k \ln(n),
    \end{equation}
    where $n$ is the number of budget steps, $\text{SSE}$ is the sum of squared errors, and $k$ is the number of parameters (including the location of breakpoints and segment slopes). The model with the minimum BIC is selected as the representative trajectory.
    \item \textbf{Preprocessing and Standardization}: Individual proxies $Z_k(t)$ are min-max standardized to the interval $[0, 1]$ prior to regression. This ensures that the SSE (and consequently the BIC) is comparable across proxies with vastly different scales, such as the magnitude of weight norms ($\sim 10^3$) versus normalized entropy ($\sim [0, 1]$).
\end{itemize}

\section{Additional Experimental Results and Analysis}
\label{sec:additional_results}

In this section, we provide additional characterization of the operational proxies and performance trajectories across various datasets (CIFAR-10, CIFAR-100, ISIC 2019), budget sizes ($b$), and representation regimes (End-to-End vs. SSL). These results demonstrate that while the specific budget counts for phase transitions are sensitive to experimental configurations, the \textit{sequential ordering} of dominance shifts, from distributional discrepancy to geometric coverage, and finally to model-driven refinement, remains an invariant structural property of the active learning process.

\subsection{Proxy Evolution and Complexity Scaling}
\label{subsec:proxies_behavior}

To analyze the structural dynamics of active learning, we visualize the evolution of our proposed operational proxies across the labeling trajectory. We provide two complementary perspectives: \textbf{absolute trajectories}, which reveal the intrinsic evolution of the manifold and model state, and \textbf{differential trajectories} ($\Delta$ vs. \textit{Random}), which isolate the marginal efficacy of specific acquisition strategies by using random sampling as a neutral baseline.

These proxies correspond directly to the decomposition terms introduced in Sec. 3, allowing us to track the shifting bottlenecks of generalization: Distribution Discrepancy (LD, FD, GC), Empirical Risk (ER), Complexity (Comp), and Confidence (Conf.). Across all benchmarks, the results reveal a consistent tripartite phase structure:
\begin{enumerate}
\item \textbf{Phase I (Data-Driven):} Early gains are dominated by reductions in discrepancy proxies (LD and FD). In this regime, performance is governed by the strategy's ability to correct sampling bias and align the labeled pool with the underlying data distribution.
\item \textbf{Phase II (Transition):} As the labeled set grows, GC and Comp. become the dominant bottlenecks. This reflects a shift from global alignment to the expansion of support across the feature manifold.
\item \textbf{Phase III (Model-Driven):} At higher budgets, ER and Conf. dictate the remaining improvements with additional secondary influence of Comp., corresponding to a refinement stage where the model focuses on resolving decision boundary ambiguities. 
\end{enumerate}

\paragraph{CIFAR-100: Complexity and Representation.}

Figures~\ref{fig:cifar100_diff}--\ref{fig:cifar100_simclr_diff} analyze proxy evolution under supervised and SimCLR representations. In the supervised setting, early gains are associated with strong reductions in discrepancy. Strategies like \textit{TypiClust} excel here by rapidly aligning the labeled subset with the data distribution. However, as the budget increases, the LD and FD gap between methods closes, while GC differences expand. Because \textit{TypiClust} continues to sample redundant near-centroid points, it fails to improve geometric coverage, allowing coverage-oriented methods like \textit{Coreset} to take the lead in the mid-budget regime. At late stages (>30k samples), we observe a distinct switch to ER and complexity dominance, where uncertainty-based methods and hybrids become the most efficient.

When utilizing \textbf{SimCLR representations} (Fig.~\ref{fig:cifar100_simclr_abs}), these dynamics are temporally compressed. Accuracy grows significantly faster because the structured feature space begins with more favorable discrepancy and coverage values. The differential analysis (Fig.~\ref{fig:cifar100_simclr_diff}) shows that while representativeness-based strategies still offer early advantages, the duration of their lead is curtailed. Since SSL already organizes the data manifold effectively, even random sampling achieves reasonable coverage, shifting the relative advantage of active learning toward later refinement stages earlier in the budget.

\begin{figure*}[!t]
\centering
\includegraphics[width=\textwidth,height=0.18\textheight]{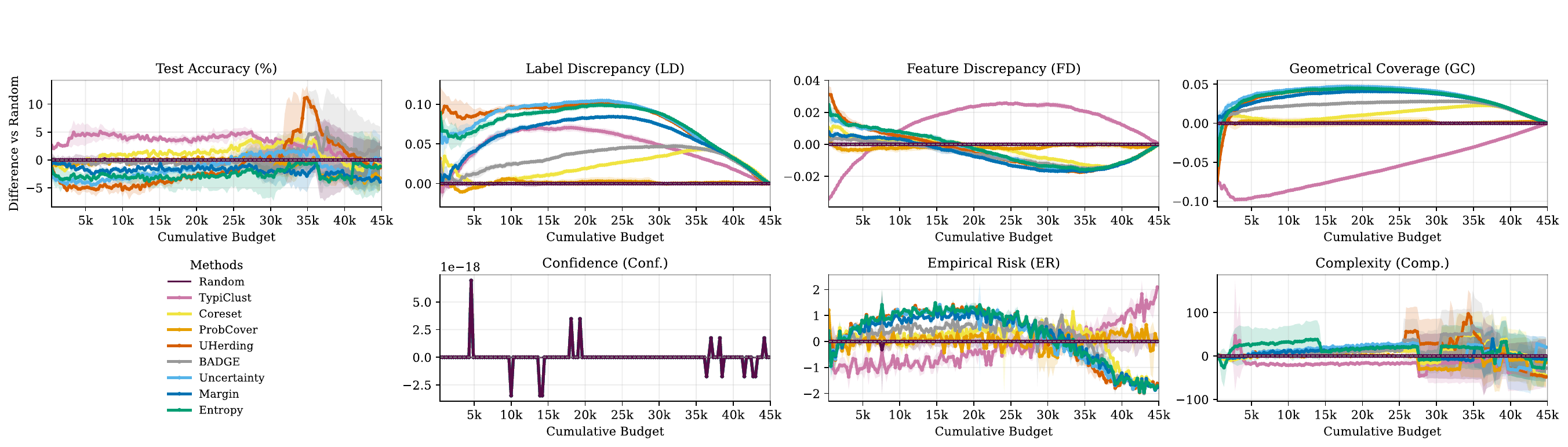}
\caption{Proxy differences relative to Random on CIFAR-100. 
Positive values indicate improvements relative to random sampling. Early gains are primarily associated with reductions in discrepancy proxies (LD and FD), indicating that representativeness dominates low-budget performance. 
As labeling progresses, improvements shift toward geometric coverage and later to empirical risk and complexity, reflecting the transition to later active learning phases. Confidence is unchanged during the labeling process.}
\label{fig:cifar100_diff}
\end{figure*}

\begin{figure*}[!t]
\centering
\includegraphics[width=\textwidth,height=0.18\textheight]{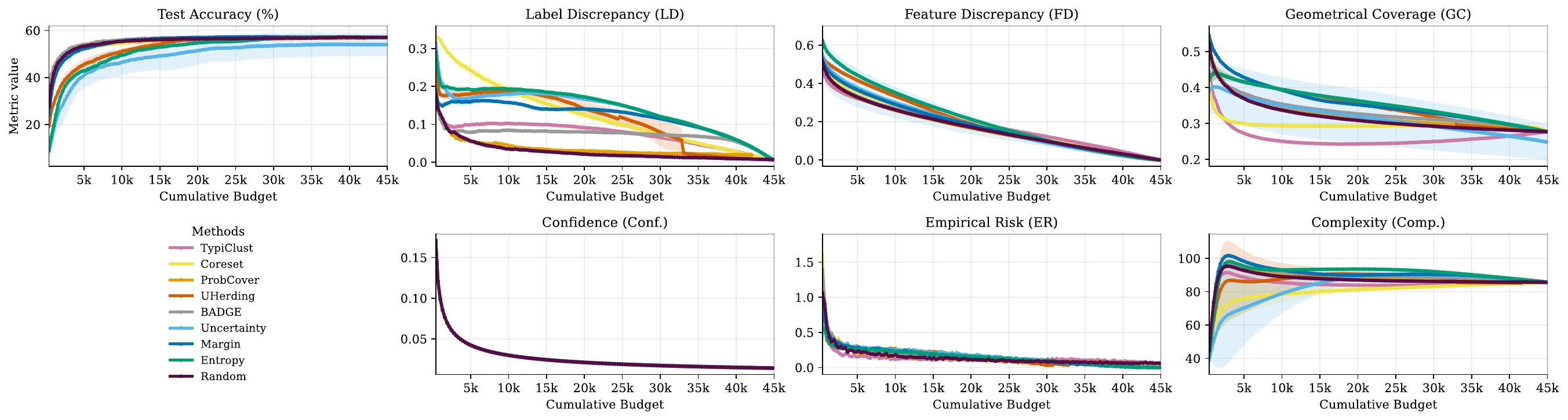}
\caption{Absolute proxy trajectories for CIFAR-100 with SimCLR features. Improved representation quality leads to faster accuracy growth and earlier stabilization of proxies, compressing the early data-driven regime.}
\label{fig:cifar100_simclr_abs}
\end{figure*}

\begin{figure*}[!t]
\centering
\includegraphics[width=\textwidth,height=0.18\textheight]{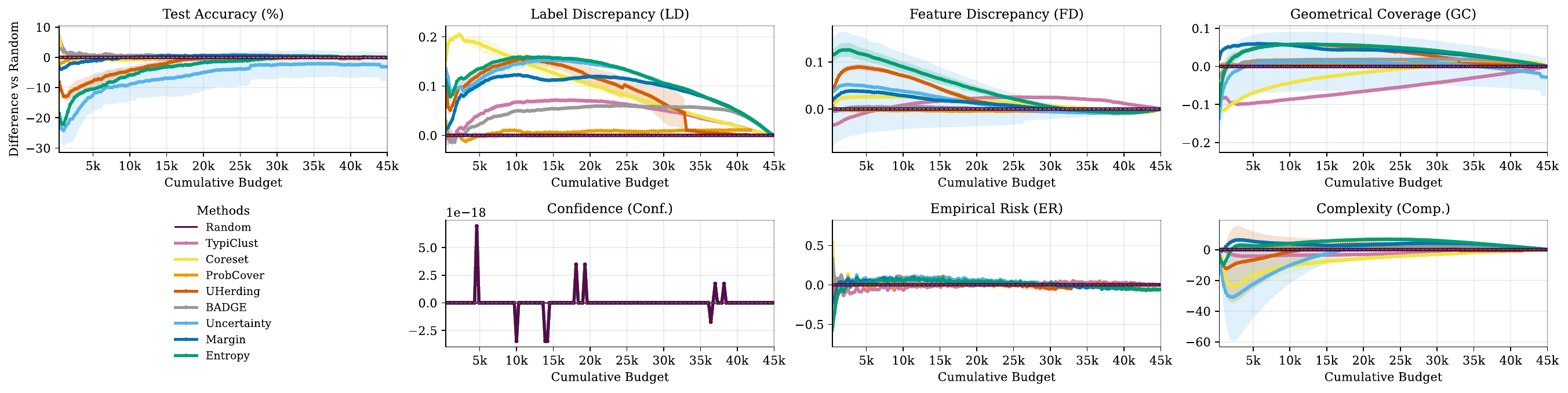}
\caption{Proxy differences relative to Random on CIFAR-100 with SimCLR features. Relative advantages become smaller and shorter-lived compared to the supervised setting, indicating that stronger representations reduce the discrepancy burden and shift improvements toward refinement.}
\label{fig:cifar100_simclr_diff}
\end{figure*}

\begin{figure*}[!t]
\centering
\includegraphics[width=\textwidth,height=0.18\textheight]{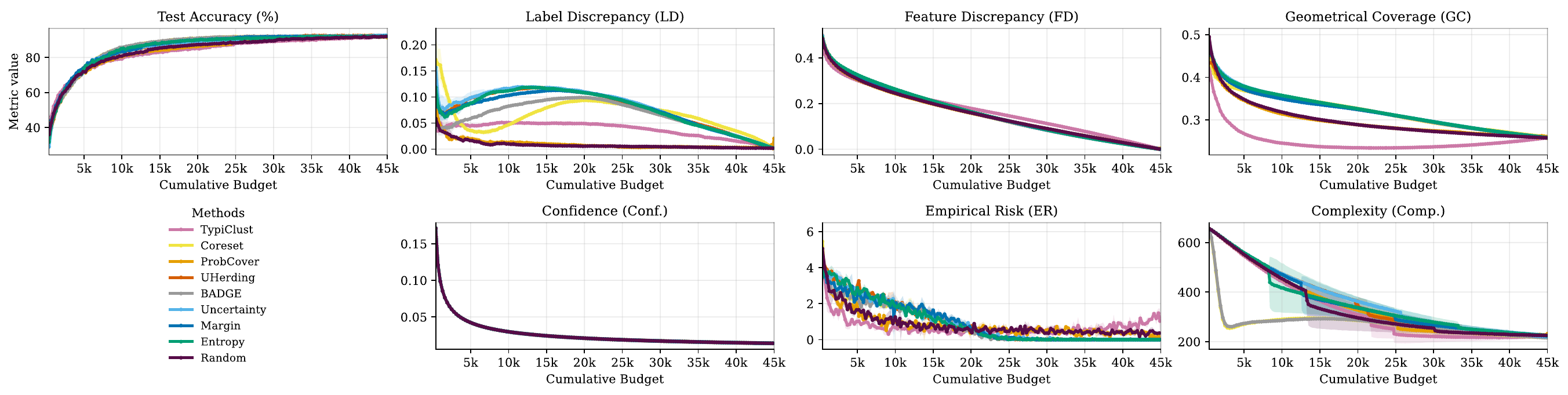}
\caption{Absolute proxy trajectories on CIFAR-10 with acquisition batch size $b=100$. Discrepancy proxies decrease rapidly, and accuracy saturates earlier than on CIFAR-100, reflecting the lower intrinsic complexity of the dataset.}
\label{fig:cifar10_b100_abs}
\end{figure*}

\begin{figure*}[!t]
\centering
\includegraphics[width=\textwidth,height=0.18\textheight]{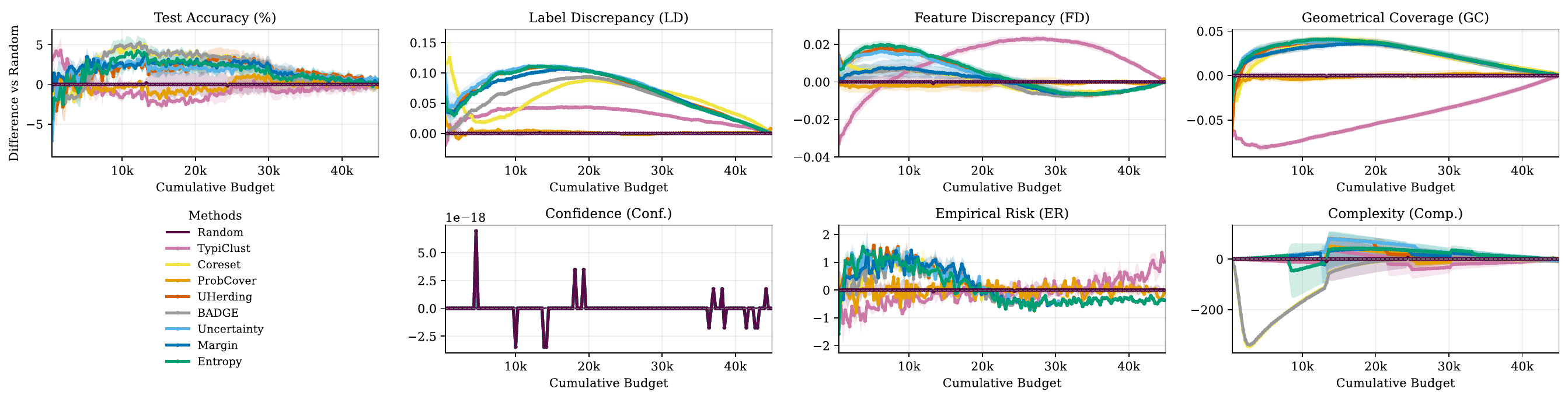}
\caption{Proxy differences relative to Random on CIFAR-10 ($b=100$). Relative method advantages are concentrated in early rounds and diminish quickly as the labeled pool becomes representative of the dataset.}
\label{fig:cifar10_b100_diff}
\end{figure*}

\begin{figure*}[!t]
\centering
\includegraphics[width=\textwidth,height=0.18\textheight]{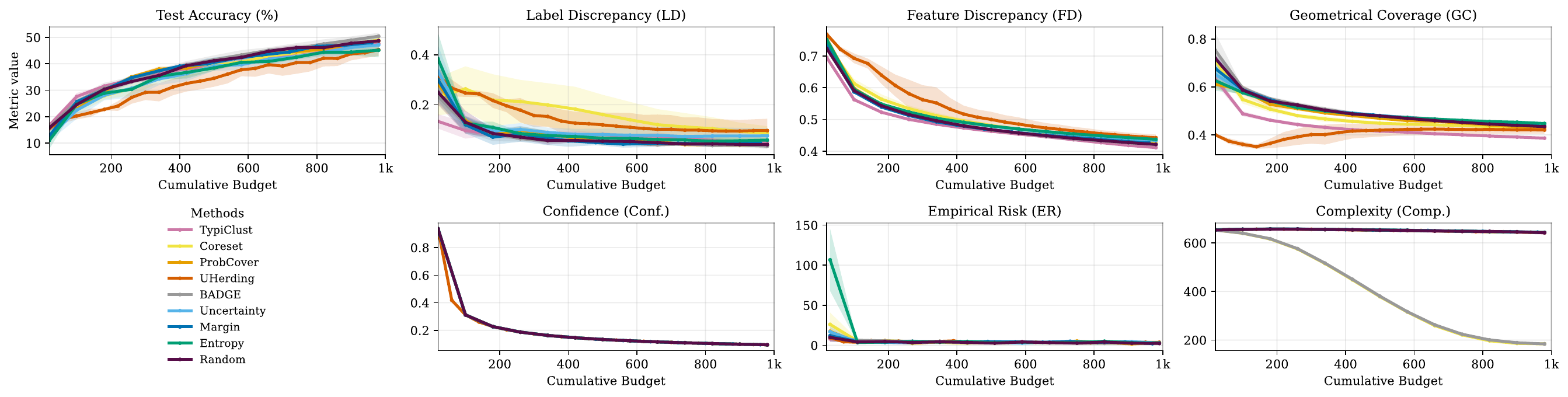}
\caption{Absolute proxy trajectories on CIFAR-10 with smaller acquisition batch size ($b=10$). The same phase structure remains visible, but transitions become smoother due to finer acquisition granularity. Notably, there is a visible correlation between proxy values and performance; for instance, \textbf{UHerding} maintains the highest values of LD, FD, and GC, which directly corresponds to its status as the worst-performing method in this regime.}  
\label{fig:cifar10_b10_abs}
\end{figure*}

\begin{figure*}[!t]
\centering
\includegraphics[width=\textwidth,height=0.18\textheight]{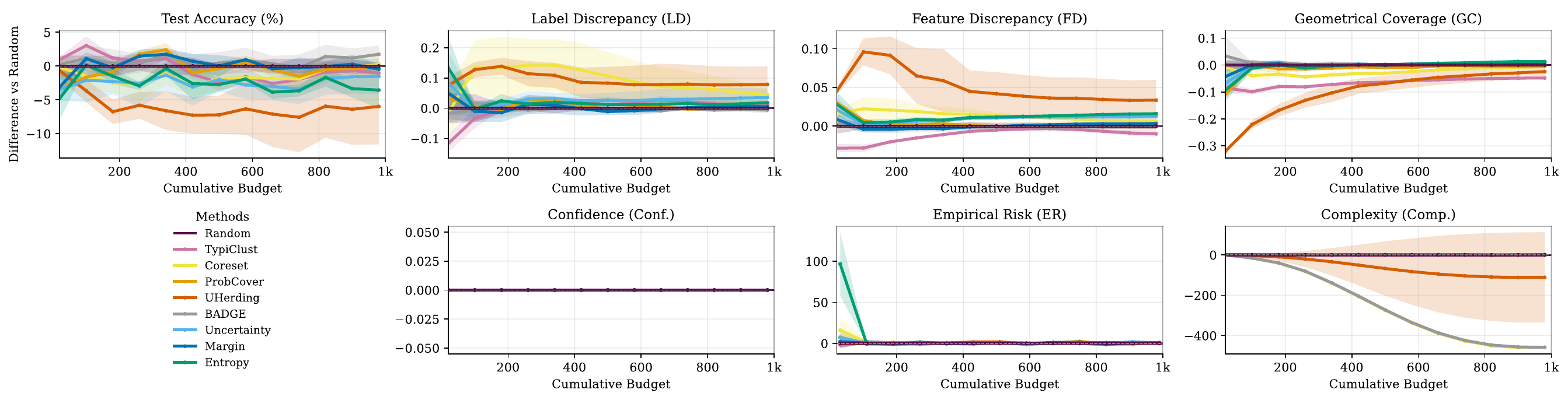}
\caption{Proxy differences relative to Random on CIFAR-10 ($b=10$). Smaller acquisition steps produce smoother deviations from the baseline while preserving the same ordering of active learning mechanisms.}
\label{fig:cifar10_b10_diff}
\end{figure*}

\paragraph{CIFAR-10: Intrinsic Complexity and Granularity.}

Figures~\ref{fig:cifar10_b100_abs}--\ref{fig:cifar10_b10_diff} examine CIFAR-10 under two acquisition scales ($b=100$ and $b=10$). With $b=100$, trajectories evolve over a shorter budget range compared to CIFAR-100; discrepancy proxies decrease rapidly, and accuracy saturates earlier. This reflects the lower intrinsic complexity of CIFAR-10, where fewer classes allow the labeled pool to represent the data distribution with fewer samples.

Reducing the batch size to $b=10$ provides a higher-resolution view of the early shifts. While the transitions appear more gradual, the qualitative ordering of mechanisms—discrepancy reduction, followed by coverage expansion, and finally empirical-risk refinement remains invariant. This confirms that our observed phase structure is not an artifact of specific acquisition steps but a fundamental property of the active learning process.

\begin{figure*}[!t]
\centering
\includegraphics[width=\textwidth,height=0.18\textheight]{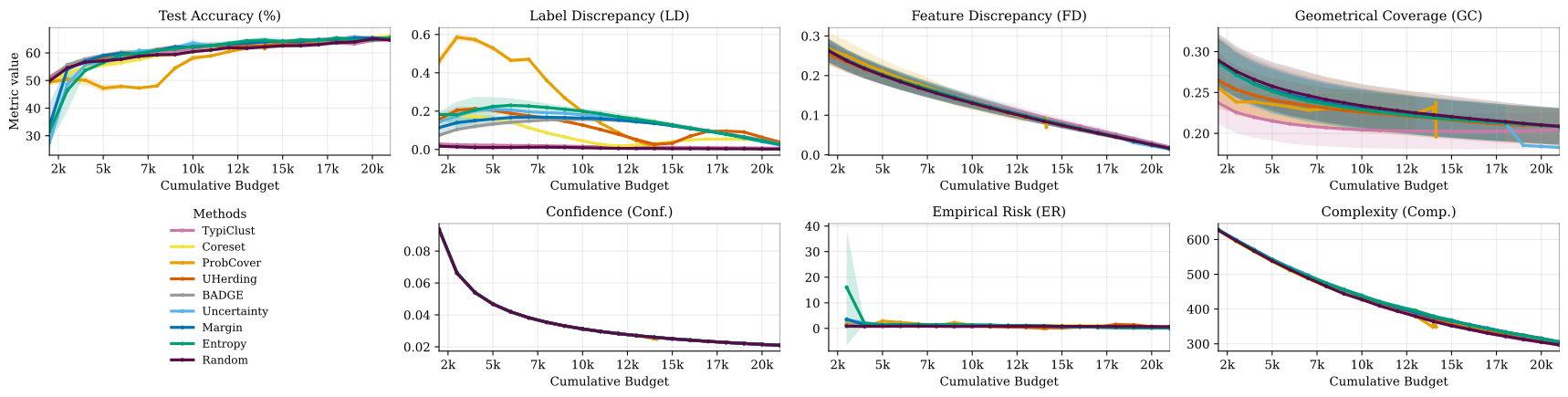}
\caption{Absolute proxy trajectories on ISIC with acquisition batch size $b=1000$. Discrepancy and coverage proxies decrease gradually, reflecting the higher variability and class imbalance of the dataset compared to CIFAR benchmarks.}
\label{fig:isic_abs_b1000}
\end{figure*}

\begin{figure*}[!t]
\centering
\includegraphics[width=\textwidth,height=0.18\textheight]{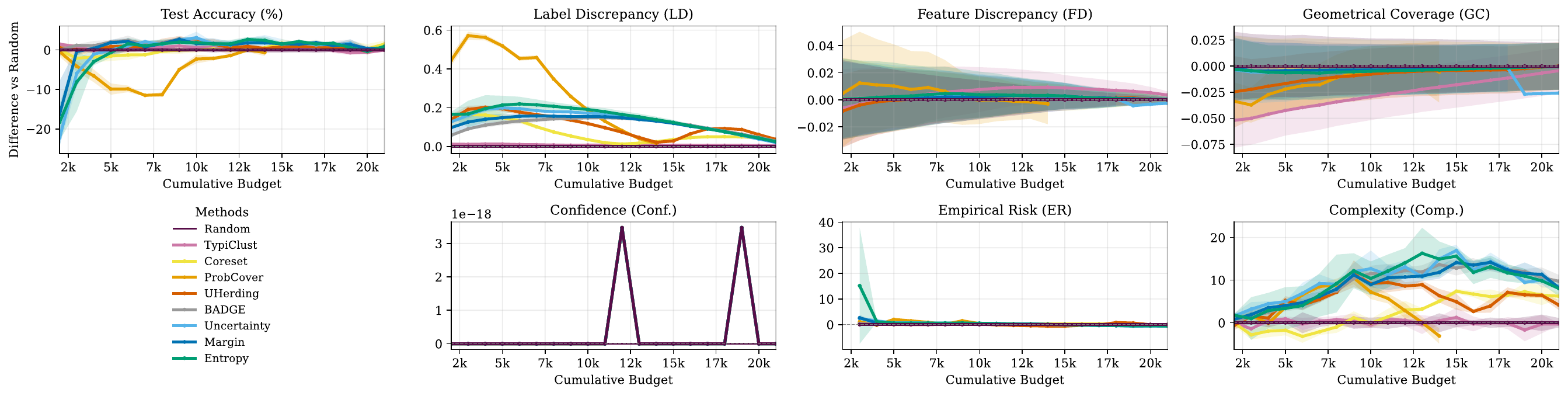}
\caption{Proxy differences relative to Random on ISIC ($b=1000$). Relative method advantages are smaller and more variable than on CIFAR, indicating that acquisition strategies struggle more to consistently improve over random sampling in this heterogeneous medical dataset.}
\label{fig:isic_diff_b1000}
\end{figure*}

\begin{figure*}[!t]
\centering
\includegraphics[width=\textwidth,height=0.18\textheight]{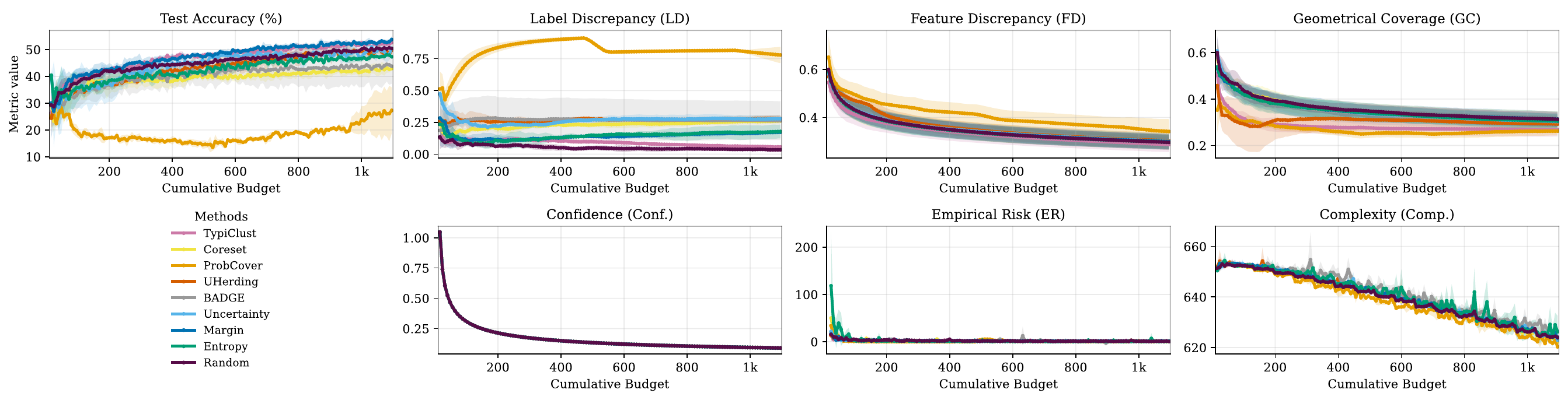}
\caption{Absolute proxy trajectories on ISIC with smaller acquisition batch size ($b=8$). The same phase structure remains visible but evolves more smoothly due to finer acquisition granularity.}
\label{fig:isic_abs_b8}
\end{figure*}

\begin{figure*}[!t]
\centering
\includegraphics[width=\textwidth,height=0.18\textheight]{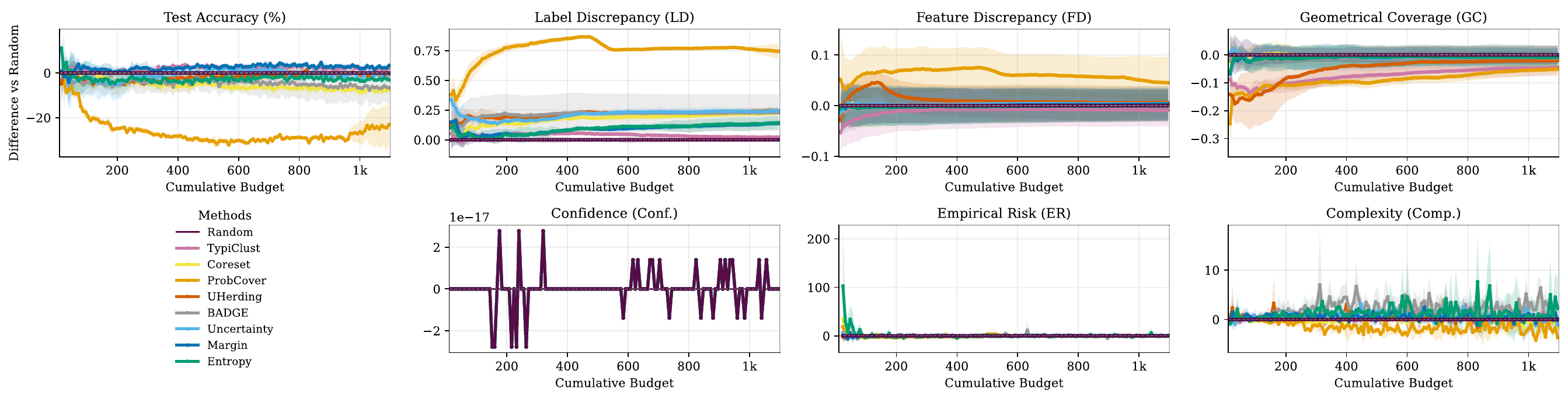}
\caption{Proxy differences relative to Random on ISIC ($b=8$). 
Smaller acquisition steps produce smoother deviations from the baseline while preserving the same qualitative ordering of active learning mechanisms.}
\label{fig:isic_diff_b8}
\end{figure*}

\paragraph{ISIC 2019: Class Imbalance and Heterogeneity.}

Figures~\ref{fig:isic_abs_b1000}--\ref{fig:isic_diff_b8} illustrate the evolution of operational proxies on the ISIC 2019 clinical dataset. Compared to the CIFAR benchmarks, the proxy dynamics in this domain evolve over a broader budget range and exhibit significantly higher variability across strategies. This behavior is a direct consequence of the extreme class imbalance and high intra-class heterogeneity inherent to skin lesion dermoscopy.

In the high-budget setting ($b=1000$), discrepancy proxies ($LD$ and $FD$) decrease gradually, indicating that aligning the labeled pool with the long-tailed distribution remains the primary generalization bottleneck for a substantial portion of the trajectory. Conversely, in the low-budget setting ($b=8$), the initial data-driven phase appears temporally compressed as proxies stabilize earlier. However, this is offset by an increase in the variability of Geometric Coverage ($GC$) and model complexity. 

This suggests that in medical domains, marginal accuracy gains are primarily derived from optimizing the \textbf{Transition Phase}, specifically, balancing the exploration of the sparse feature manifold with complexity management to avoid overfitting on majority classes. Overall, the ISIC results confirm that while the transition phase occupies the majority of the labeling process in complex datasets, the structural sequence, progressing from data-driven alignment to geometric exploration and finally uncertainty-based refinement, remains invariant. This reinforces our central claim that the sequential bottleneck structure persists even under noisy, real-world data distributions.

\subsection{Qualitative Analysis via t-SNE Visualization}
\label{supp:tsne_visuals}

To illustrate the spatial dynamics underlying our framework, we visualize the selection behavior of different active learning strategies using t-SNE projections. Figures~\ref{fig:tsne_comparison_cifar10} and \ref{fig:tsne_comparison_imagenet50} demonstrate how acquisition methods navigate the embedding space as the cumulative budget B increases. Points are colored by ground-truth class, with black points indicating the existing labeled pool and red points denoting the newly acquired samples.

\begin{figure*}[!t]
\centering
\includegraphics[width=0.87\textwidth,height=0.40\textheight]{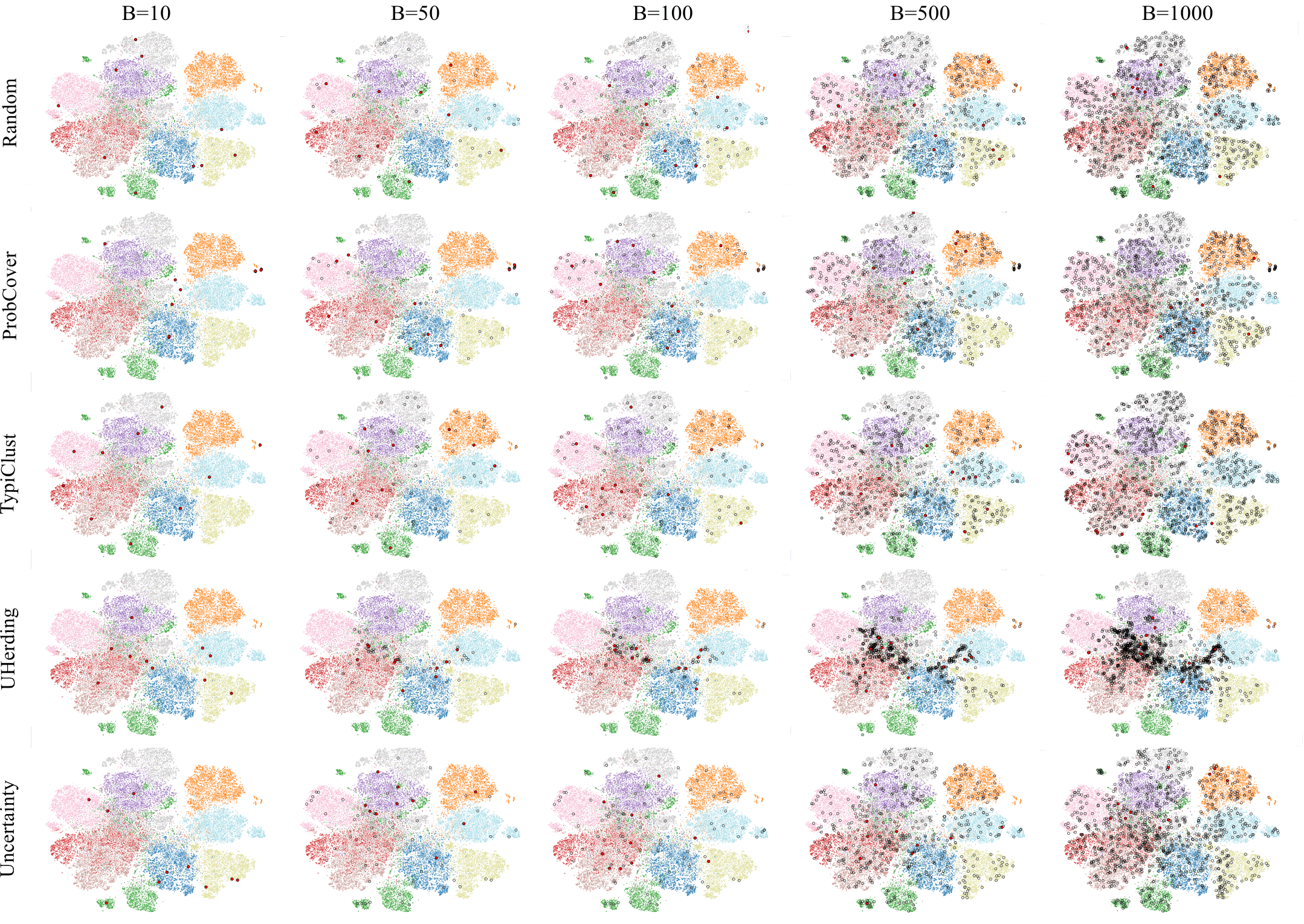}
\caption{t-SNE visualization of sample selection behavior across different AL strategies on CIFAR-10 at increasing annotation budgets. Points are projected using SimCLR embeddings and colored by ground-truth class. Black points indicate previously labeled samples, while red points denote newly selected samples in the current round. Representation-based methods (e.g., TypiClust) emphasize class diversity early on, while uncertainty-based strategies focus on ambiguous regions.}
\label{fig:tsne_comparison_cifar10}
\end{figure*}

\begin{figure*}[t!]
\centering
\includegraphics[width=0.9\textwidth]{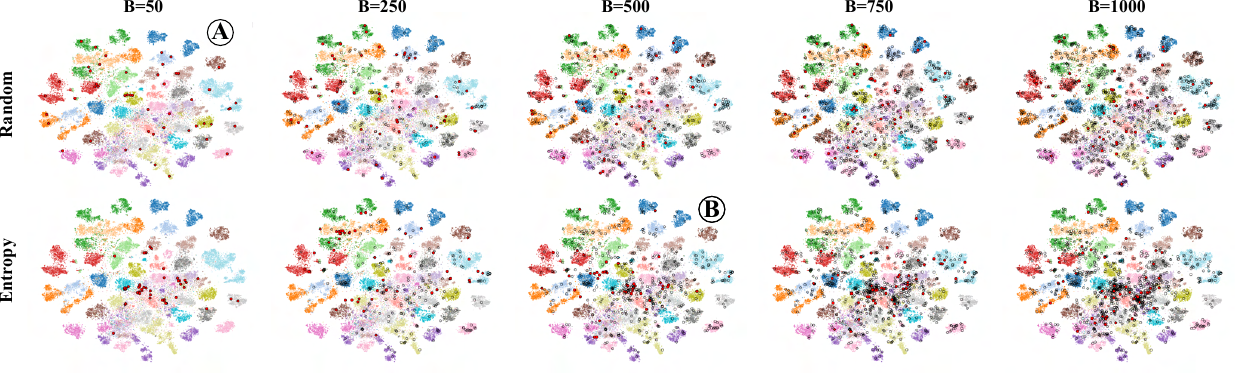}
\includegraphics[width=0.9\textwidth]{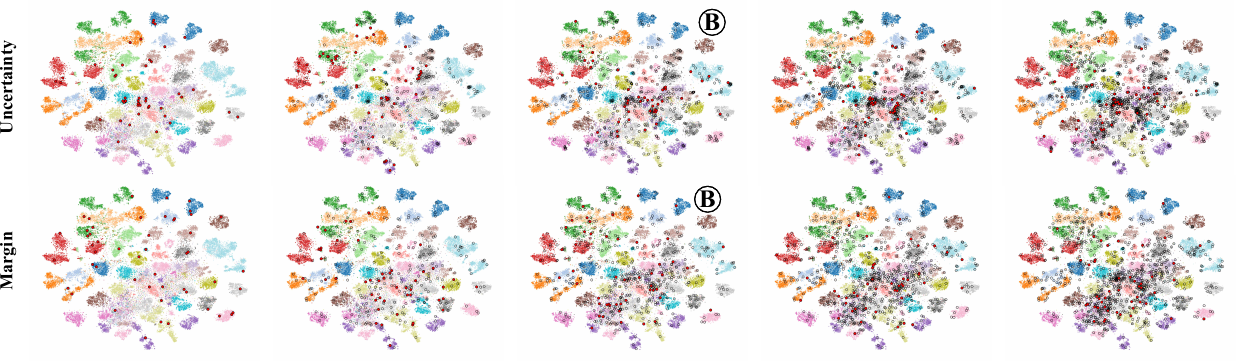}
\includegraphics[width=0.9\textwidth]{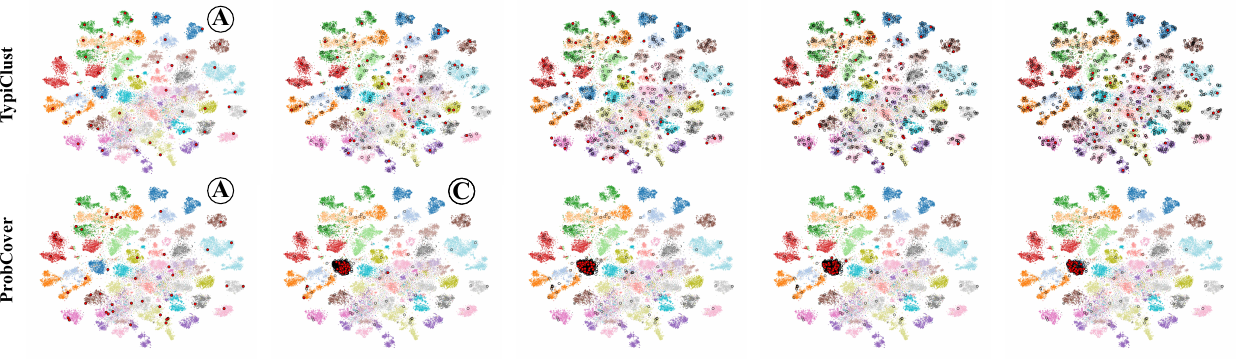}
\includegraphics[width=0.9\textwidth]{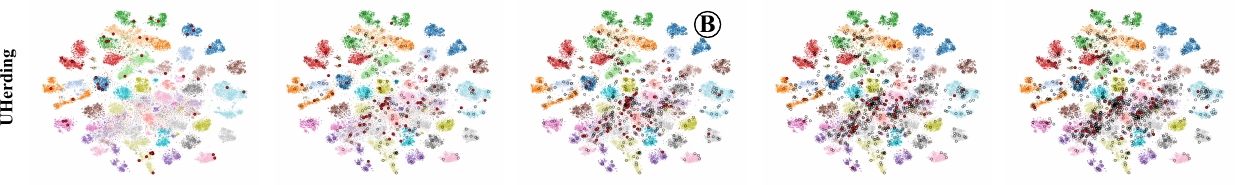}
\caption{t-SNE visualization of sample selection behavior on ImageNet-50 at increasing annotation budgets. Points are projected using MoCov3 embeddings and colored by ground-truth class. Representativeness-based methods achieve broad class coverage early (see \textcircled{A}), while uncertainty-based strategies focus on ambiguous regions (see \textcircled{B}). ProbCover occasionally oversamples dense regions due to its predefined coverage radius (see \textcircled{C}).}
\label{fig:tsne_comparison_imagenet50}
\end{figure*}

\paragraph{CIFAR-10.}

Figure~\ref{fig:tsne_comparison_cifar10} shows the evolution of sample selection behavior on CIFAR-10 at increasing annotation budgets. In the low-budget regime (e.g., $B=10$), representativeness-based strategies such as TypiClust promote strong class-level diversity by selecting representative samples from each cluster of the feature space. In contrast, uncertainty-based strategies such as Uncertainty sampling and UHerding focus on ambiguous or overlapping regions, often concentrating selections near cluster boundaries and leaving some classes underrepresented.

As the budget increases, the behavior of different strategies diverges. TypiClust continues prioritizing dense regions of the embedding space, which improves class coverage but may underexplore uncertain regions near cluster boundaries. Uncertainty-driven methods increasingly sample from ambiguous regions but often fail to achieve broad spatial coverage. Random sampling provides relatively uniform coverage across the embedding space, acting as a strong baseline due to its balanced exploration. ProbCover exhibits intermediate behavior, combining spatial coverage with local density but occasionally oversampling densely populated subregions, leading to redundant selections.

\paragraph{ImageNet-50.}

Figure~\ref{fig:tsne_comparison_imagenet50} presents the same analysis on ImageNet-50 using MoCov3 embeddings. The learned representation produces relatively well-separated clusters corresponding to the 50 classes, though some residual overlap remains in the central region of the embedding space.

At small cumulative budgets (e.g., $B=50$), representativeness-based methods such as TypiClust and ProbCover effectively cover representative regions of the feature space (see \textcircled{A}), achieving broad class-level coverage. In contrast, uncertainty-based strategies tend to concentrate on ambiguous or noisy regions, often selecting multiple points from the same areas of the embedding space.

As the annotation budget increases (see \textcircled{B}), these trends persist. Representativeness approaches continue emphasizing spatial coverage, while uncertainty-based methods concentrate on informative boundary regions. However, a limitation of the default ProbCover configuration is visible in \textcircled{C}: the predefined coverage radius can be too large for dense clusters, causing repeated selections within a single region and reducing overall coverage.

Overall, the t-SNE visualizations highlight how different acquisition strategies implicitly prioritize representativeness, uncertainty, and coverage. The relative importance of these principles changes as the labeling budget increases, supporting the proposed phase-based interpretation of AL dynamics.

\subsection{Regime Identification via Segmented Regression}
\label{subsec:regime_identification}

\begin{figure*}[!t]
\centering
\includegraphics[width=\textwidth]{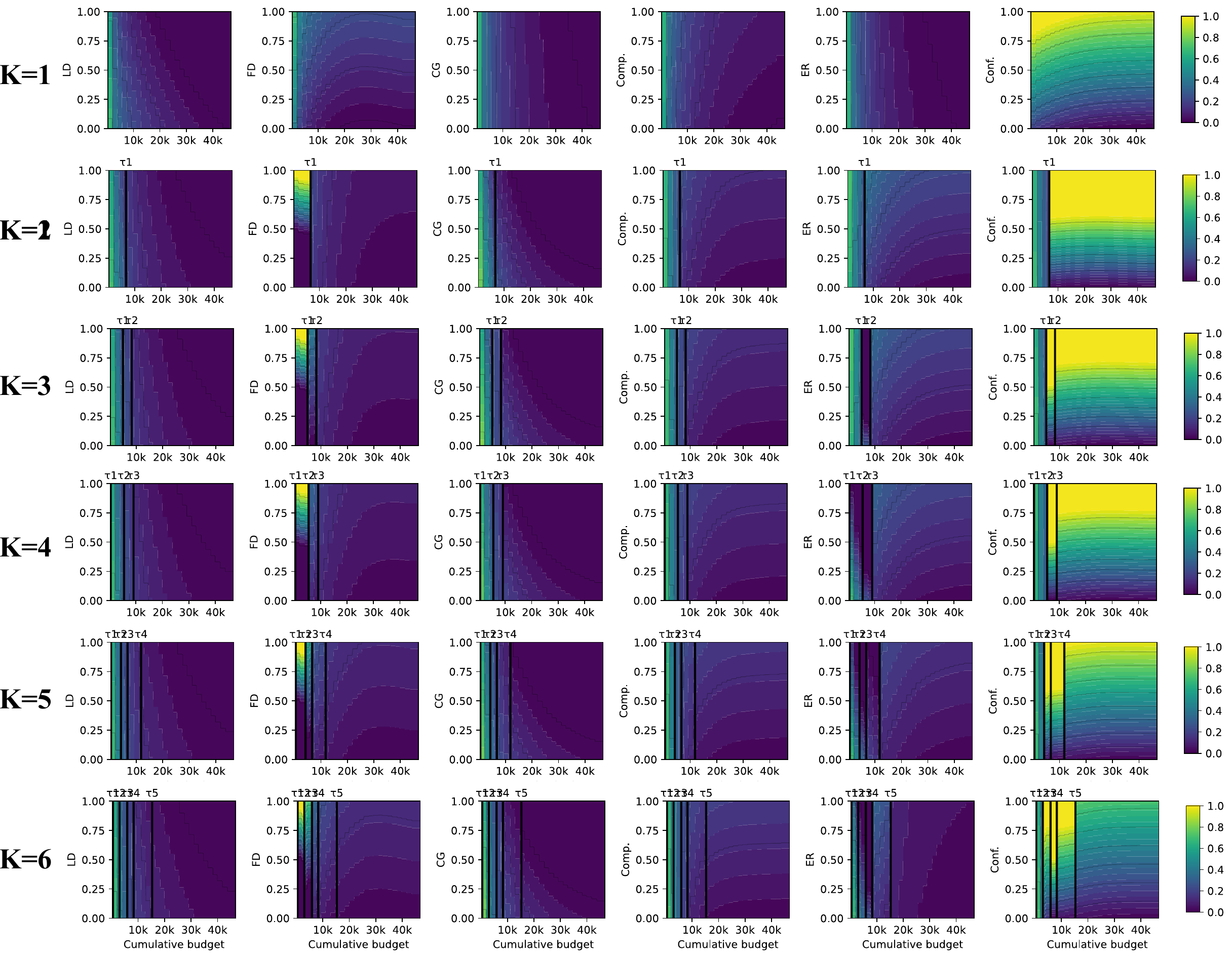}
\caption{Segmented proxy importance across the labeling trajectory on CIFAR-10 for different numbers of segments $K$. Heatmaps show normalized proxy importance as a function of cumulative labeling budget. Vertical lines indicate estimated regime breakpoints $\tau$.}
\label{fig:cifar10_pieces}
\end{figure*}

\begin{figure*}[t]
\centering
\includegraphics[width=\textwidth]{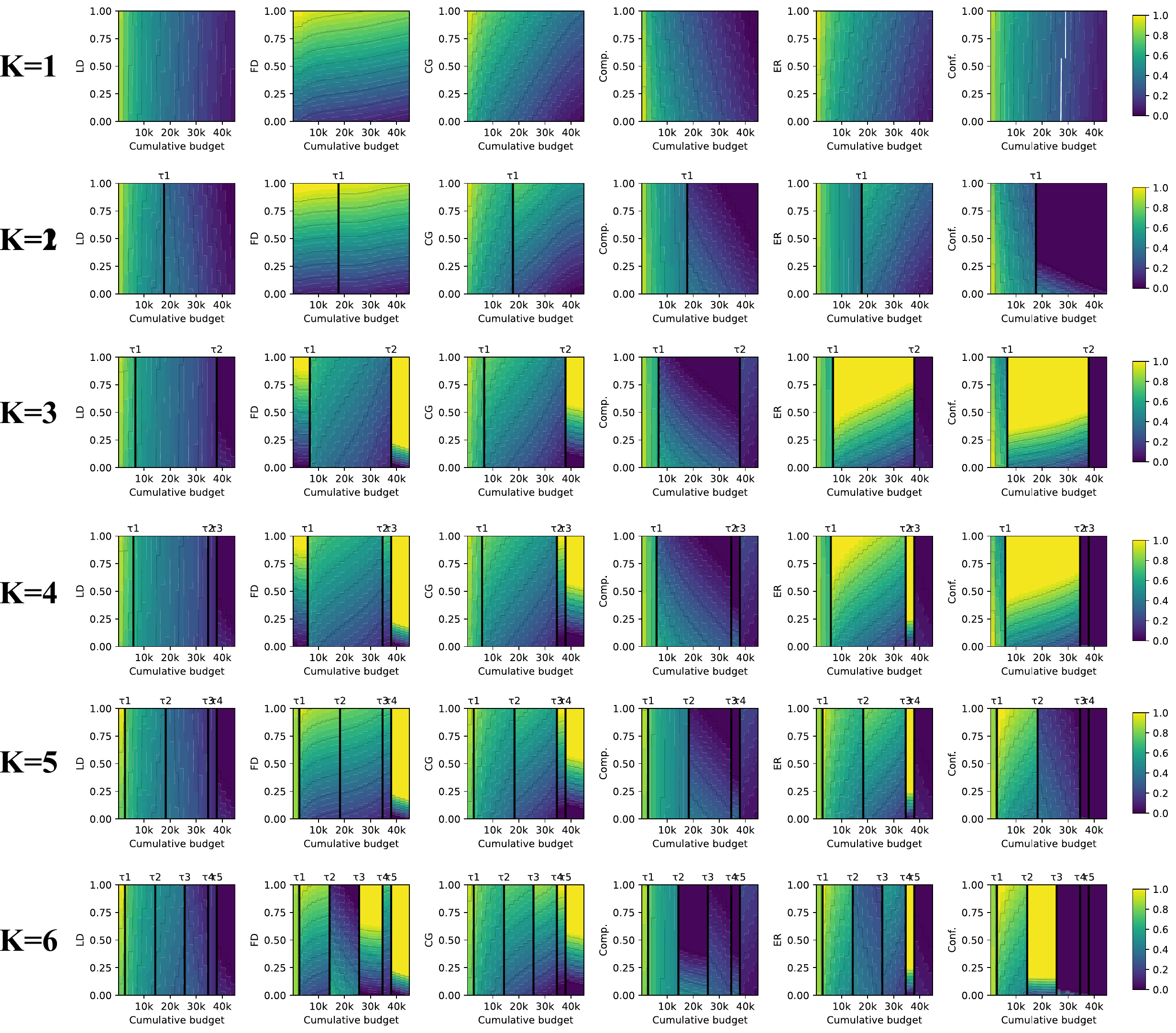}
\caption{Segmented regression analysis on CIFAR-100 showing proxy importance across budgets for increasing numbers of segments $K$. As $K$ increases, distinct breakpoints emerge that reveal transitions between discrepancy-dominated, coverage-driven, and refinement regimes.}
\label{fig:cifar100_pieces}
\end{figure*}

\begin{figure*}[t]
\centering
\includegraphics[width=\textwidth]{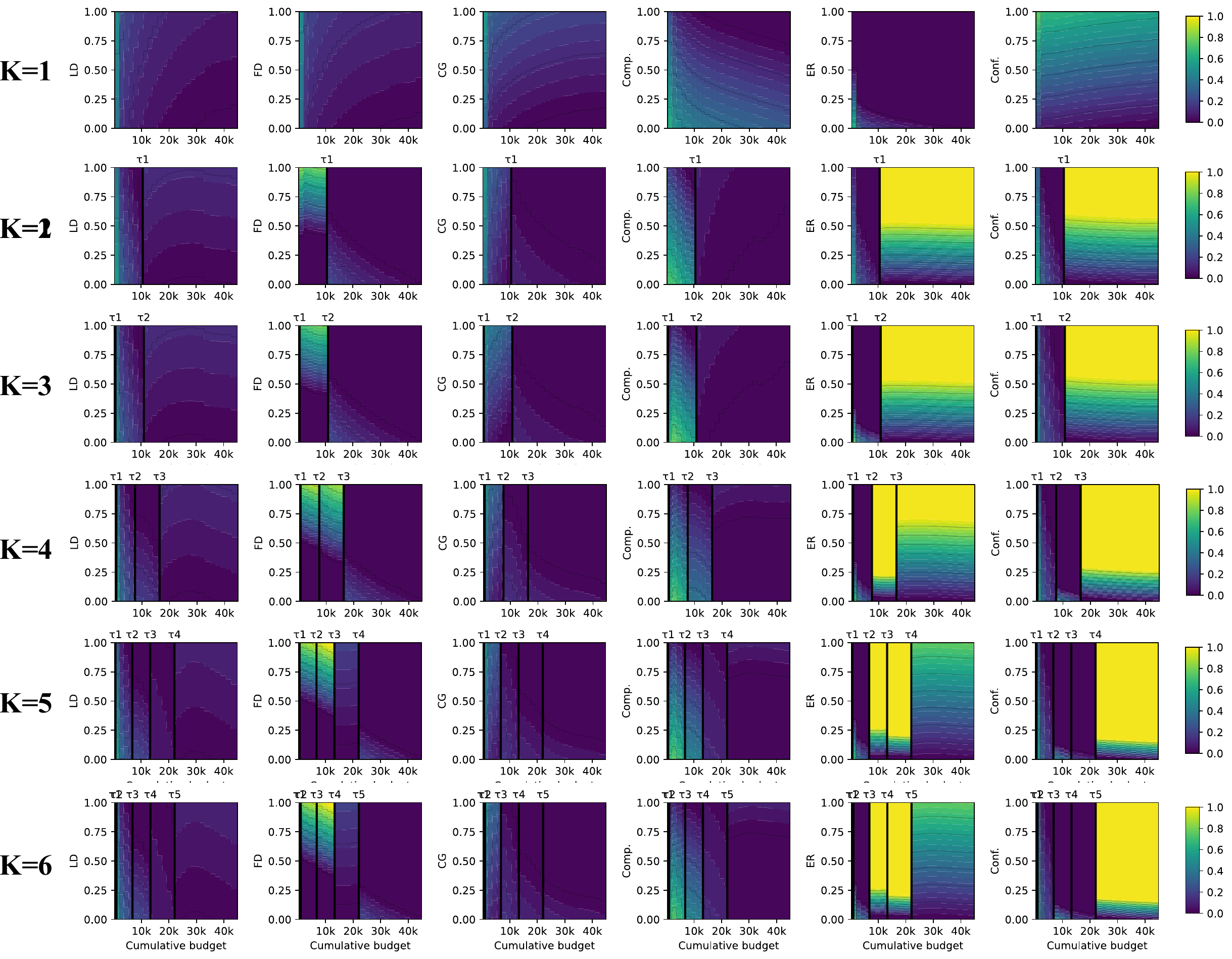}
\caption{Segmented proxy importance on CIFAR-100 using SimCLR representations. Stronger pretrained features shift regime transitions earlier in the labeling trajectory while preserving the same three-phase structure observed in the supervised representation.}
\label{fig:cifar100_simclr_pieces}
\end{figure*}

\begin{figure*}[t]
\centering
\includegraphics[width=\textwidth]{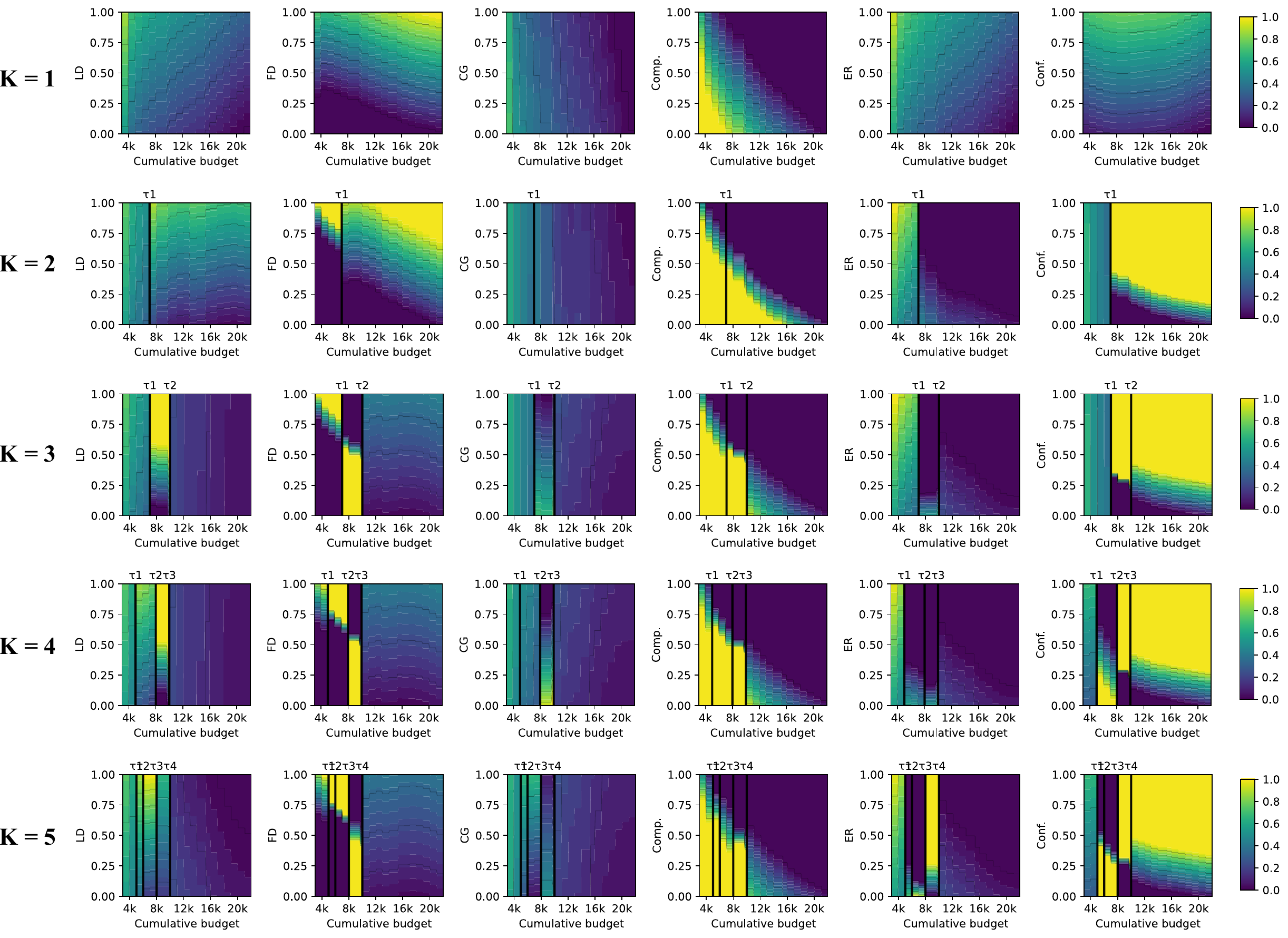}
\caption{Segmented regression analysis on ISIC. Compared to CIFAR datasets, regime transitions occur later and exhibit greater variability due to higher visual diversity and class imbalance characteristic of medical imaging datasets.}
\label{fig:isic_pieces}
\end{figure*}

To further analyze the phase structure of active learning dynamics, we perform segmented regression over the labeling trajectory, examining how the relative importance of the proposed proxies evolves as the annotation budget increases. This procedure operationalizes the theoretical shifts in the generalization bound by identifying transit breakpoints in the empirical data. 

\paragraph{Proxy Importance Evolution.}

Figures~\ref{fig:cifar10_pieces}, \ref{fig:cifar100_pieces}, \ref{fig:cifar100_simclr_pieces}, and \ref{fig:isic_pieces} visualize the estimated proxy importance across the labeling trajectory for varying numbers of segments $K$. Each heatmap displays the normalized contribution of a given proxy, with vertical lines indicating the estimated regime breakpoints $\tau$ identified by the segmented regression. 

This visualization specifically highlights the \textit{sensitivity} of the generalization proxies. Because each proxy is normalized within its own empirical range, the heatmaps reveal how subtle fluctuations in one metric can exert a disproportionately high impact on the generalization bound compared to larger absolute changes in others. Such dynamics are often obscured when viewing raw numerical values or simple rankings.

For $K=1$, the regression model assumes a single stationary regime, resulting in smooth, averaged importance trends that fail to capture the underlying structural shifts. However, as $K$ increases, distinct breakpoints emerge, revealing the physical transitions between competing learning mechanisms:

\begin{itemize}
    \item \textbf{Early Stages (Phase I):} Consistently dominated by discrepancy-related proxies (LD and FD). This emphasizes the critical requirement of correcting sampling bias and aligning the labeled subset with the underlying data distribution during the initial ``cold-start'' period.
    \item \textbf{Intermediate Stages (Phase II):} Geometric Coverage (GC) and Model Complexity become the primary drivers. As the labeled pool expands across the feature manifold, the active bottleneck shifts from global alignment to the local density and spatial coverage of the representation space.
    \item \textbf{Late Stages (Phase III):} Empirical Risk (ER) and confidence-based terms dominate. This corresponds to a refinement phase where acquisition focuses on resolving residual decision boundary errors once the training set has become sufficiently representative of the manifold support.
\end{itemize}

In CIFAR-100 with a high budget, we observe a secondary resurgence in the importance of coverage and complexity terms. This suggests that once the model achieves the primary three stages of learning, it pivots back toward further manifold exploration and fine-grained refinement. In this specific context, the $K=4$ model captures these nuances more clearly. Furthermore, across all datasets, we observe that Label Discrepancy (representing pure representativeness) gradually collapses and never regains dominance relative to GC. Feature Discrepancy (FD) exhibits a ``mixed'' effect, acting as a transitional proxy that bridges the gap between distributional alignment (LD) and geometric coverage (GC).

Importantly, the three-regime model ($K=3$) consistently captures these core transitions across all benchmarks while maintaining model parsimony. While larger values of $K$ introduce additional breakpoints, they yield diminishing returns in interpretability, suggesting that our proposed tripartite taxonomy provides a robust and theoretically grounded description of AL dynamics.

\paragraph{Impact of Self-Supervised Representations.}

As shown in Figure~\ref{fig:cifar100_simclr_pieces}, the use of SimCLR representations preserves the three-phase structure but significantly shifts the regime transitions earlier in the labeling trajectory. The stronger pretrained features minimize the initial discrepancy burden, allowing the transition to coverage-based and model-driven regimes to occur at a fraction of the budget required for supervised representations.

\subsection{Regression Model Validation}
\label{subsec:regression_validation}

\begin{figure*}[t]
\centering
\includegraphics[width=\textwidth]{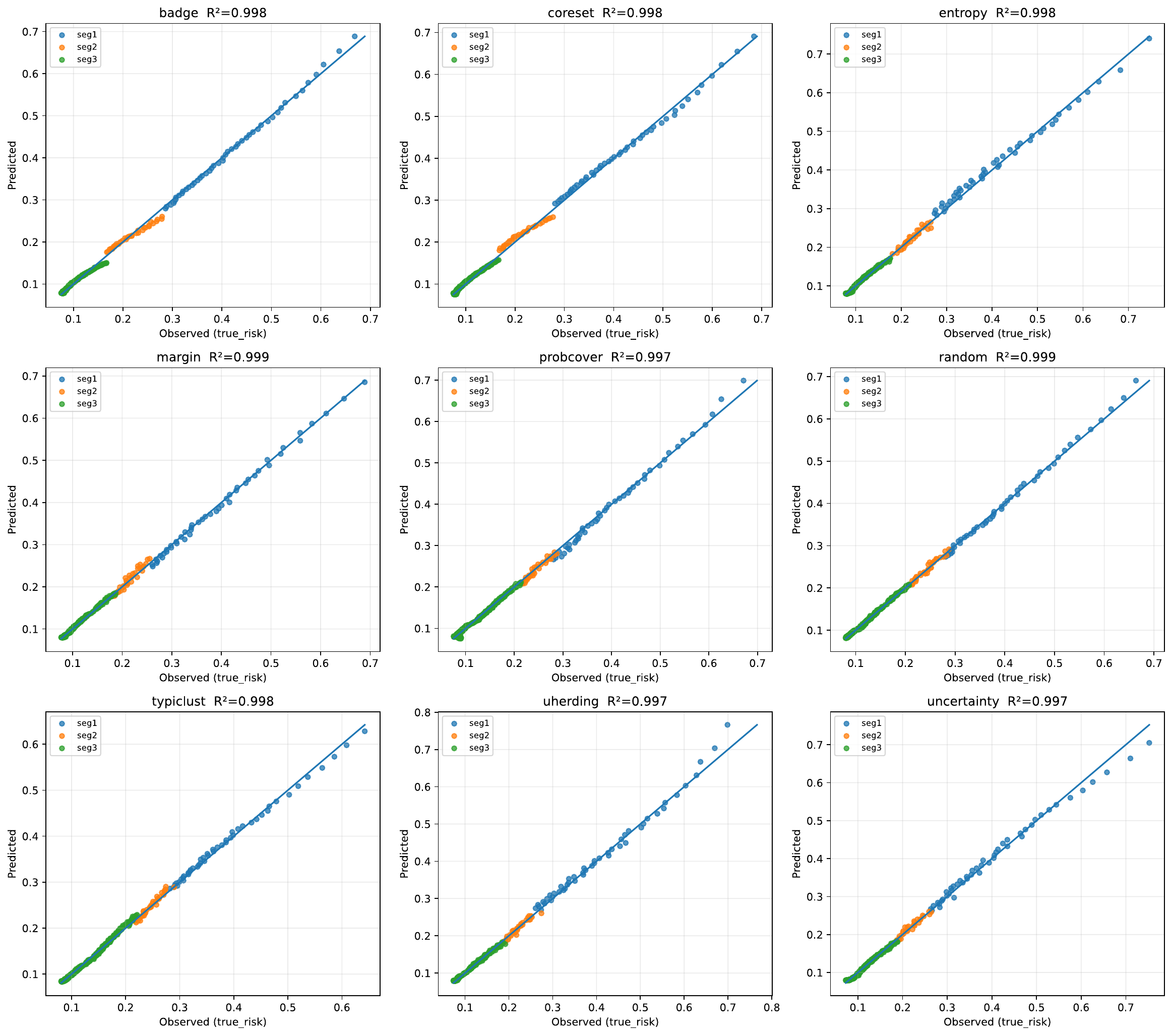}
\caption{Observed versus predicted risk under the segmented regression model ($K=3$) on CIFAR-10. Points are colored according to the regime assigned by the model. High $R^2$ values across acquisition strategies indicate strong predictive accuracy.}
\label{fig:cifar10_regression}
\end{figure*}

\begin{figure*}[t]
\centering
\includegraphics[width=\textwidth]{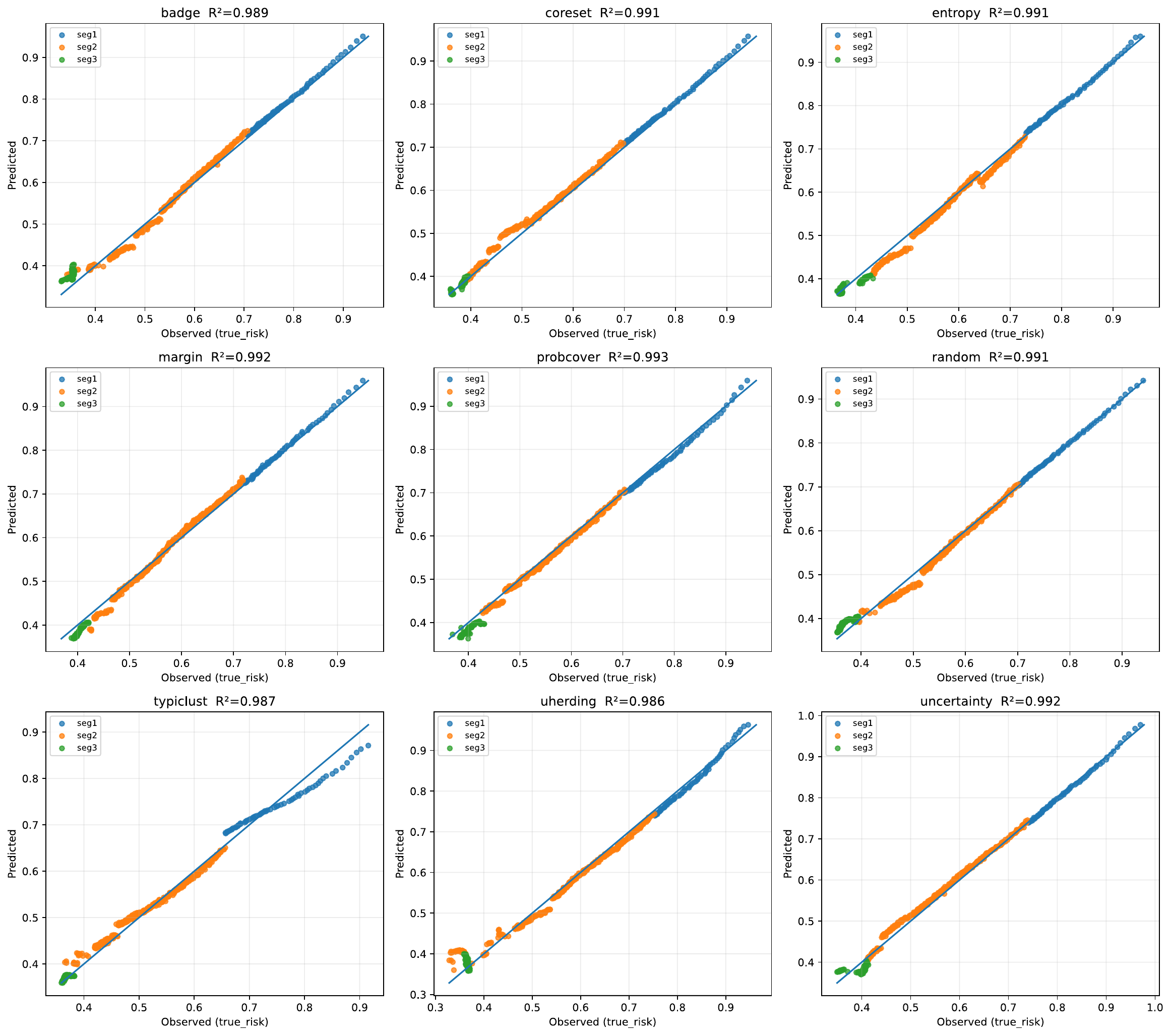}
\caption{Regression validation for CIFAR-100 under the $K=3$ segmented model. Predicted and observed risks align closely across acquisition strategies, supporting the validity of the identified regime structure.}
\label{fig:cifar100_regression}
\end{figure*}

\begin{figure*}[t]
\centering
\includegraphics[width=\textwidth]{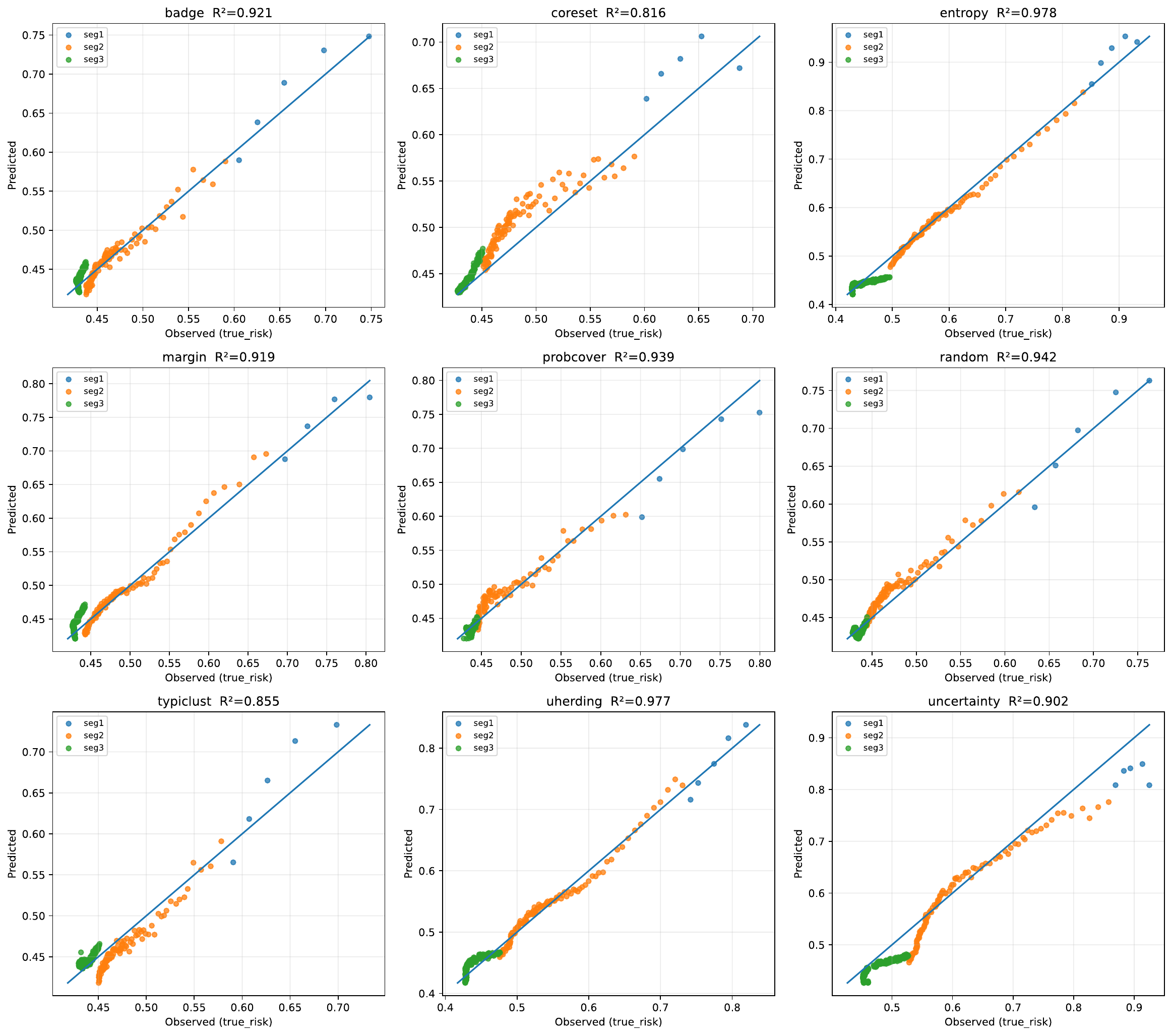}
\caption{Regression validation for CIFAR-100 with SimCLR features under the $K=3$ segmented model. The strong agreement between predicted and observed risk confirms that the regime decomposition remains consistent when stronger representations are used.}
\label{fig:cifar100_simclr_regression}
\end{figure*}

\begin{figure*}[t]
\centering
\includegraphics[width=\textwidth]{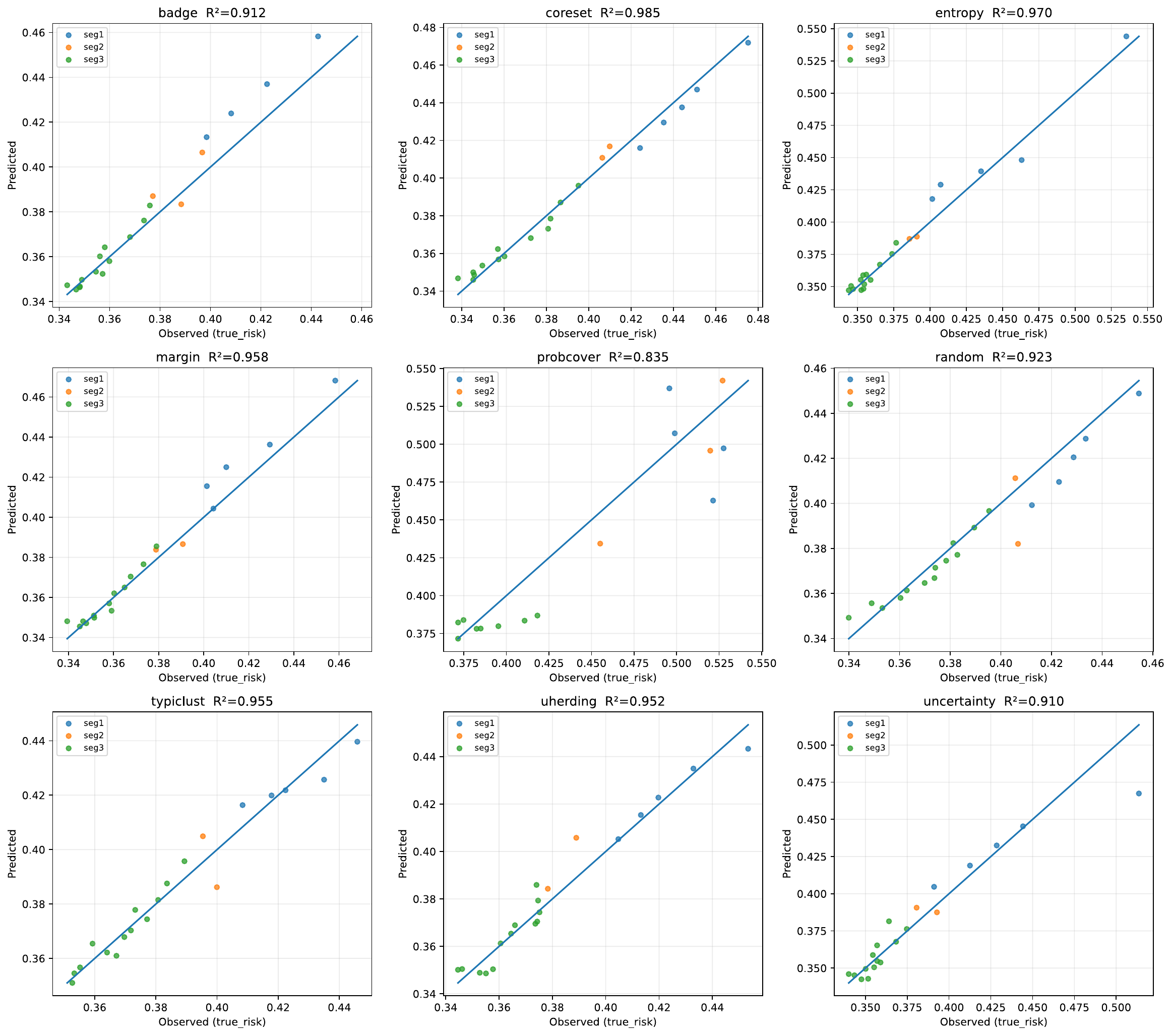}
\caption{Observed versus predicted risk for ISIC under the $K=3$ segmented regression model. Despite higher dataset variability, the segmented model captures the main structure of the risk trajectory across acquisition strategies.}
\label{fig:isic_regression}
\end{figure*}

To validate the fidelity of the K=3 segmented regression model, we compare the predicted risk values against the true observed risk across all acquisition strategies in Figures~\ref{fig:cifar10_regression}, \ref{fig:cifar100_regression}, \ref{fig:cifar100_simclr_regression}, and \ref{fig:isic_regression}. Across datasets and acquisition strategies, the predicted values closely align with the observed risk, as evidenced by the high $R^2$ values. This strong agreement confirms that the selected operational proxies effectively capture the underlying drivers of generalization performance:

\begin{itemize}
\item \textbf{CIFAR-10 \& CIFAR-100:} The regression model demonstrates near-perfect alignment, validating that the tripartite decomposition of the generalization bound is a precise fit for natural image datasets where manifold structures are relatively well-defined.
\item \textbf{CIFAR-100 (SimCLR):} Under SimCLR representations, the regression model shows a less linear alignment compared to the supervised setting. Since pretrained features accelerate early phase transitions, the intensity of proxy changes is higher over a shorter budget. Still, the separation of the segments remains distinct, and the model maintains high predictive fidelity.
\item \textbf{ISIC 2019:} Despite the extreme heterogeneity, noise, and class imbalance characteristic of medical imaging, the K=3 model captures the primary structural components of the risk trajectory. While these plots exhibit higher variance (Fig.~\ref{fig:isic_regression}), the fundamental phase transitions remain statistically significant and well-modeled.
\end{itemize}

\paragraph{Qualitative Strategy Alignment.}

The color-coding of points in the validation plots, representing the identified regimes, reveals that different acquisition strategies traverse these regimes at varying speeds. Representativeness-based strategies, such as \textbf{TypiClust}, are shown to minimize risk most effectively during the discrepancy-dominated regime (Phase I). In contrast, uncertainty-based methods typically lag in performance until the model enters the refinement stage (Phase III).

Overall, these results provide rigorous empirical support for the existence of distinct learning regimes. The close alignment between predicted and observed risk validates the central hypothesis of this work: AL performance is not a monolithic process but a sequence of structurally distinct bottlenecks governed by the evolving geometry of the labeled pool and the model's hypothesis space.

\subsection{Alignment between Operational Proxies and Acquisition Efficacy}
\label{subsec:proxy_relation_to_methods}

To demonstrate the predictive utility of our framework, we visualize the temporal alignment between proxy dominance and the empirically best-performing acquisition strategy. This Regime Mapping serves as a direct validation of our core hypothesis: that the efficiency of an AL method is determined by the alignment between its inductive bias and the active generalization bottleneck.

\paragraph{Visualization Protocol.}

Figures~\ref{fig:cifar10_regime_map}, \ref{fig:cifar100_regime_map}, \ref{fig:cifar100_simclr_regime_map}, and \ref{fig:isic_regime_map} integrate three critical layers of information:

\begin{enumerate}
\item \textbf{The Proxy Heatmap}: Displays the normalized importance of LD, FD, GC, ER, Comp, and Conf across the budget.
\item \textbf{The Top Strip (Dominant Proxy)}: Identifies which specific component of the generalization bound is the primary bottleneck at any given step $t$.
\item \textbf{The Bottom Strip (Dominant Method)}: Identifies the strategy achieving the highest test accuracy at budget $t$.
\end{enumerate}
Vertical lines indicate the regime breakpoints $\tau$ derived from the segmented regression analysis.

\paragraph{CIFAR-10 and CIFAR-100: Mechanistic Hand-offs.}

In the natural imaging benchmarks (Figs.~\ref{fig:cifar10_regime_map} and \ref{fig:cifar100_regime_map}), we observe a near-perfect temporal synchronization between theoretical bottlenecks and empirical winners.

\begin{itemize}
\item \textbf{Phase I (Discrepancy-Dominant)}: The top strip is dominated by LD and FD. Correspondingly, the bottom strip identifies representativeness-driven methods like \textbf{TypiClust} as the winners. This confirms that at low budgets, the primary gain comes from aligning the labeled sample distribution with the population.
\item \textbf{Phase II (Coverage-Dominant)}: As the budget crosses $\tau_1$, the bottleneck shifts toward GC. Here, coverage-oriented strategies (e.g., \textbf{Coreset}) or hybrids begin to outperform purely representative ones.
\item \textbf{Phase III (Refinement-Dominant)}: At high budgets, ER and complexity become the active proxies. The bottom strip shifts to uncertainty-based methods (\textbf{Entropy}, \textbf{Margin}), which refine the decision boundary in regions where the manifold is already well-covered.
\end{itemize}

While both datasets follow this pattern, we observe higher variability in dominance for CIFAR-10, particularly in the transition phase. In contrast, CIFAR-100 exhibits a significantly more stable representativeness phase, likely due to the higher number of classes requiring a more prolonged alignment period.

\paragraph{SSL Compression in CIFAR-100.}

The SimCLR regime map (Fig.~\ref{fig:cifar100_simclr_regime_map}) provides compelling evidence for the impact of representation quality. The initial discrepancy phase is drastically shortened compared to the supervised version. Because the pretrained feature space is already well-structured, the GC and ER proxies ``activate'' at much lower label counts, causing the bottom strip to transition into coverage and uncertainty methods far earlier in the trajectory. Notably, we observe a resurgence of coverage importance after 25k samples, where GC regains dominance; this is mirrored in the performance of \textbf{BADGE} and \textbf{Coreset}, which regain dominance during this secondary exploration stage.

\paragraph{ISIC 2019: Robustness under Imbalance.}

Figures~\ref{fig:isic_regime_map} and \ref{fig:isic_b8_regime_map} analyze the medical imaging domain. Due to the high visual diversity and extreme class imbalance, the discrepancy-dominated Phase I persists for a larger portion of the budget. Interestingly, for the fine-grained $b=8$ acquisition (Fig.~\ref{fig:isic_b8_regime_map}), we observe a secondary resurgence of GC importance after the initial refinement phase. This suggests that in complex, noisy domains, the model may oscillate between refining known boundaries and needing to ``re-explore'' the manifold to find minority-class samples missed during earlier rounds.

\paragraph{Synthesis of Results.}

The strong correspondence between the Top Strip (Theory) and Bottom Strip (Empirical Performance) across all datasets confirms that AL dynamics are not stochastic but governed by the evolving geometry of the labeled pool. These maps provide the first systematic evidence that AL performance can be deconstructed into a sequence of deterministic mechanism shifts, allowing for the potential design of ``regime-aware'' adaptive acquisition strategies.

\begin{figure*}[t]
\centering
\includegraphics[width=\textwidth]{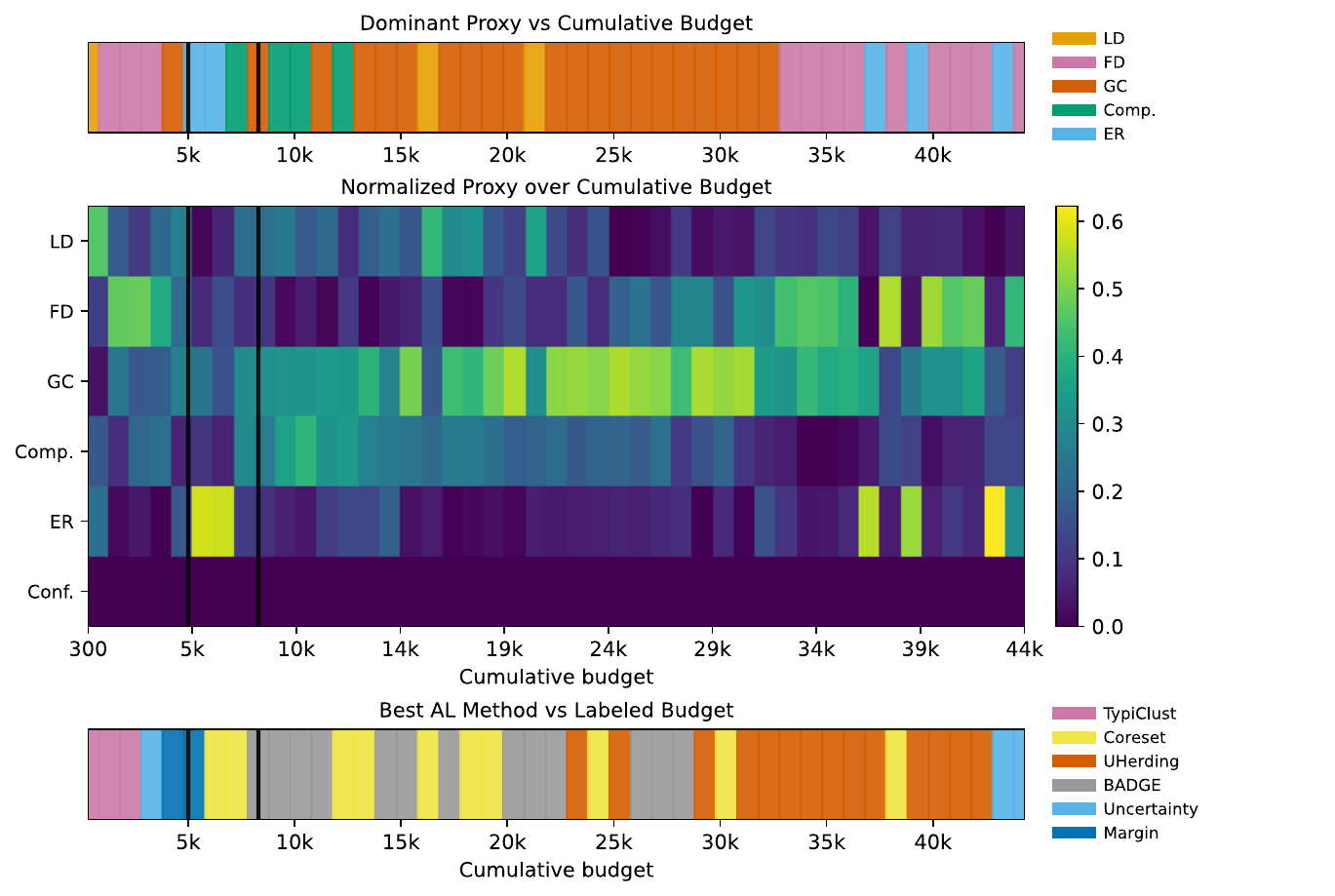}
\caption{Proxy–method alignment across the labeling trajectory on CIFAR-10. The top strip shows the dominant proxy at each labeling budget, while the heatmap displays normalized proxy importance. The bottom strip indicates the best-performing active learning method. Vertical lines correspond to regime breakpoints identified by segmented regression.}
\label{fig:cifar10_regime_map}
\end{figure*}

\begin{figure*}[t]
\centering
\includegraphics[width=\textwidth]{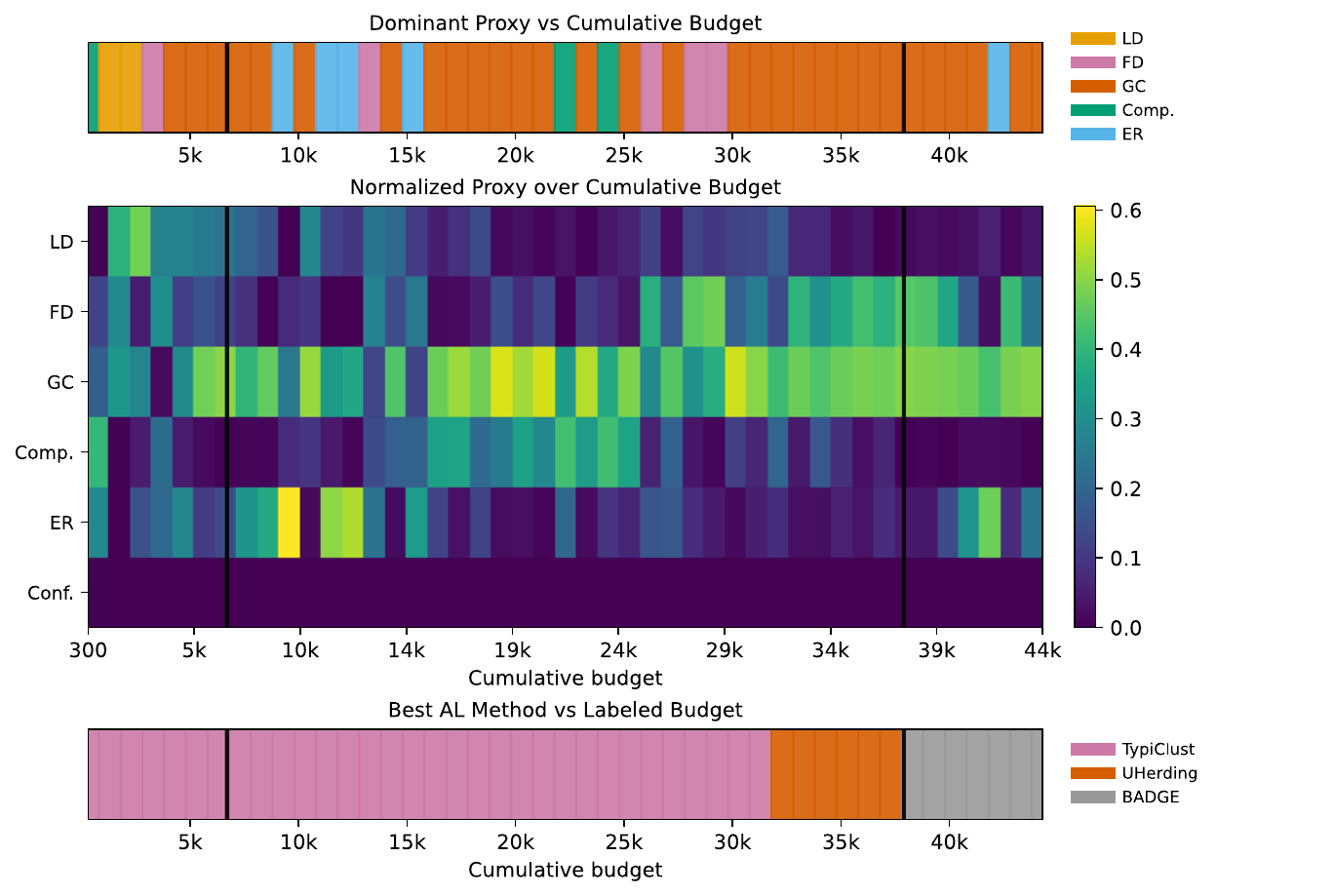}
\caption{Proxy–method alignment on CIFAR-100. Early stages are dominated by discrepancy proxies, while coverage becomes more important at intermediate budgets. The best-performing acquisition strategies evolve accordingly.}
\label{fig:cifar100_regime_map}
\end{figure*}

\begin{figure*}[t]
\centering
\includegraphics[width=\textwidth]{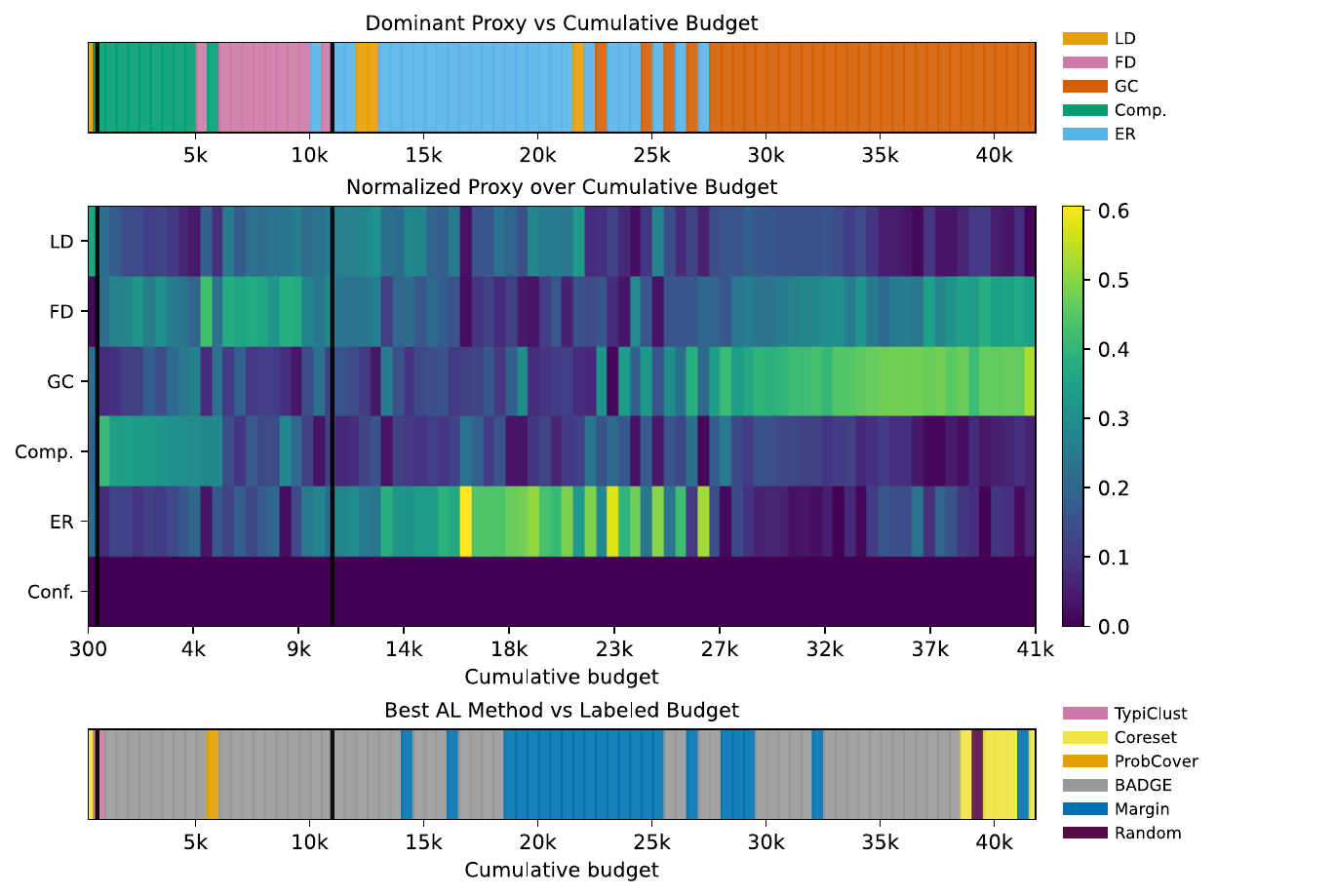}
\caption{Proxy–method alignment for CIFAR-100 using SimCLR representations. Improved feature quality shifts regime transitions earlier in the labeling trajectory while preserving the same proxy–method alignment pattern.}
\label{fig:cifar100_simclr_regime_map}
\end{figure*}

\begin{figure*}[t]
\centering
\includegraphics[width=\textwidth]{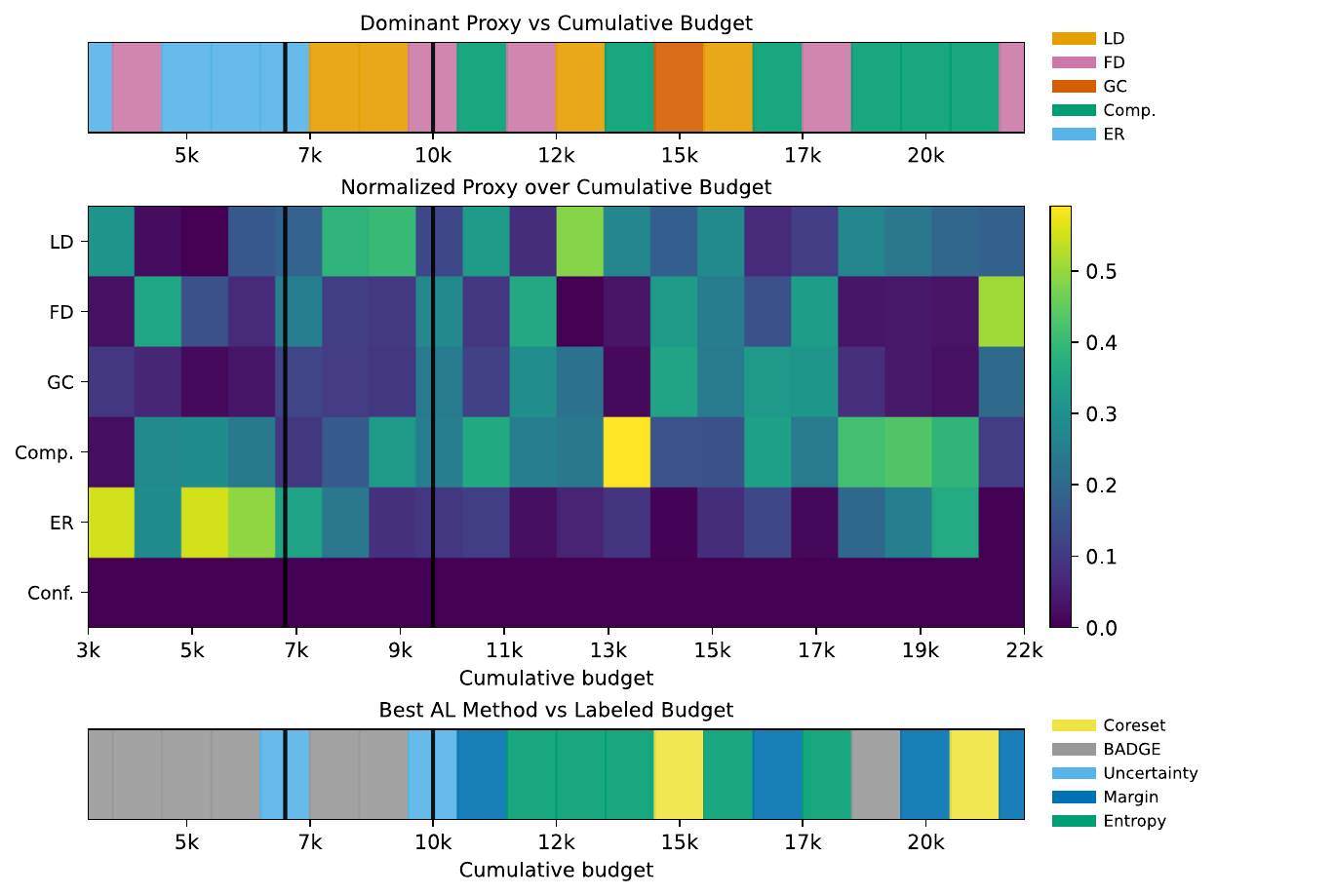}
\caption{Proxy–method alignment on ISIC with $b=1000$. 
Compared to natural image datasets, proxy transitions occur later and exhibit higher variability due to dataset complexity and class imbalance.}
\label{fig:isic_regime_map}
\end{figure*}

\begin{figure*}[!t]
\centering
\includegraphics[width=\textwidth]{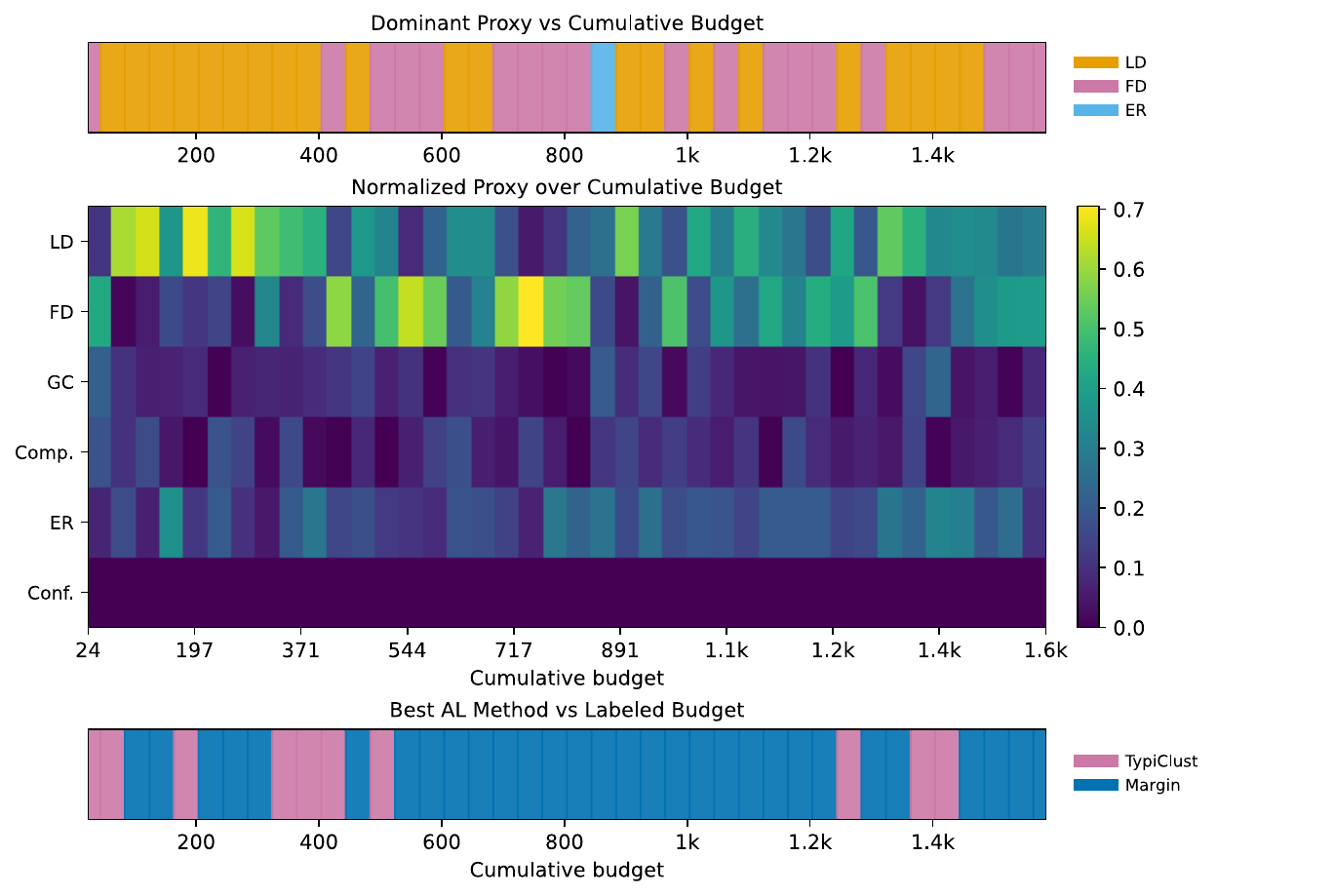}
\caption{Proxy–method alignment on ISIC with $b=8$ presenting early phase detailed proxy importance. Compared to natural image datasets, proxy transitions occur later and exhibit higher variability due to dataset complexity and class imbalance.}
\label{fig:isic_b8_regime_map}
\end{figure*}

\subsection{Cold-Start Ablation and Computational Efficiency}
\label{supp:cold_start}

Active learning typically suffers from a \emph{cold-start} problem: when labeled data is scarce, model uncertainty estimates are unreliable, and acquisition strategies struggle to identify informative points on an unorganized manifold. Self-supervised learning (SSL) provides a natural mitigation by structuring the feature space prior to the introduction of any supervision.

\paragraph{Mitigating the Cold-Start with SSL.}

\begin{figure*}[ht]
\centering
\includegraphics[width=\textwidth]{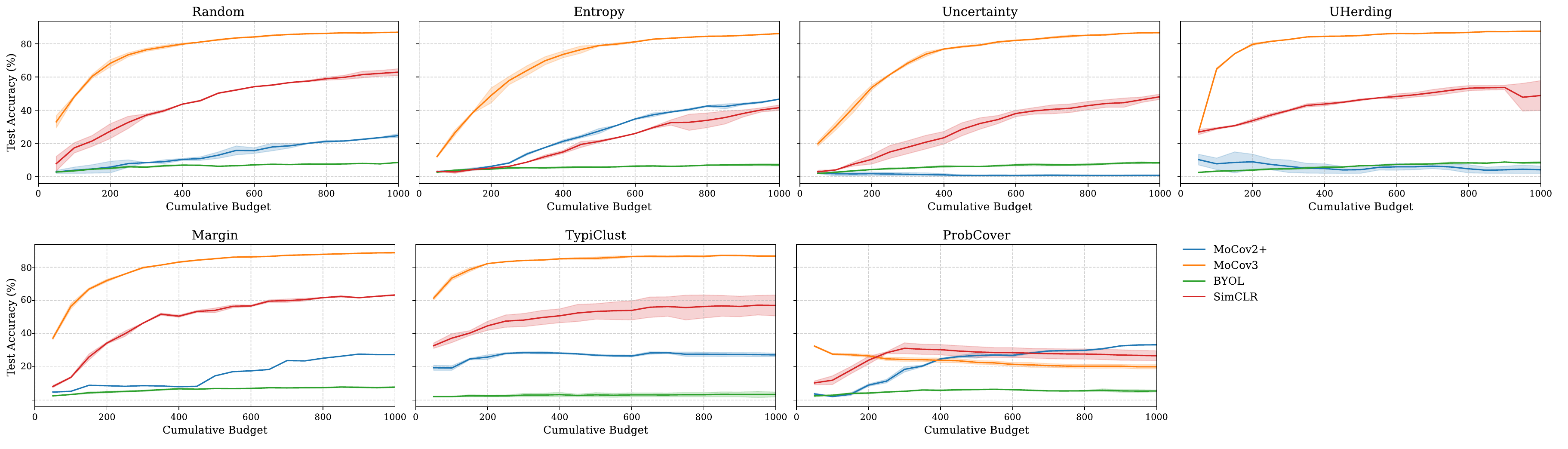}
\caption{Low-budget performance of active learning methods on ImageNet-50 using different self-supervised embeddings. Structured SSL representations significantly improve early-stage performance across most methods. Representativeness-based strategies such as TypiClust benefit most from the semantically clustered feature space, while methods relying on fixed geometric assumptions (e.g., ProbCover) may experience degraded performance due to mismatched coverage radii.}
\label{fig:ssl_comparison}
\end{figure*}

SSL representations organize samples into semantically meaningful clusters based on intrinsic data properties. This reduces the reliance on noisy, early-stage uncertainty estimates and allows acquisition strategies to operate on a pre-aligned feature manifold. 

As shown in Figure~\ref{fig:ssl_comparison}, SSL features significantly improve early-stage accuracy across nearly all methods on ImageNet-50. The gains are most pronounced for representativeness-based strategies like \textbf{TypiClust}, which leverage the improved semantic clustering. Conversely, methods with rigid geometric priors, such as \textbf{ProbCover}, can sometimes experience degraded performance if their predefined coverage radius does not align with the condensed geometry of the SSL embedding space. 

Furthermore, the embedding quality influence is strongly evident. With \textbf{MoCov3}, nearly every method achieves higher accuracy at a lower budget, whereas with \textbf{BYOL}, some methods are barely able to exceed baseline performance. This suggests that the utility of AL is not just a function of the acquisition strategy, but is heavily coupled with the fidelity of the underlying representation manifold.

\paragraph{Computational Speedups and Scalability.}

Beyond accuracy gains, SSL representations enable massive computational savings by fundamentally altering the training workflow. By utilizing a structured, frozen feature space, we provide the model with a foundation of high-level semantic knowledge from the onset. Consequently, the downstream classifier is significantly easier to train, as it optimizes over a pre-aligned manifold rather than learning features from scratch. This strategy allows us to bypass the resource-intensive process of end-to-end backpropagation through deep architectures in every active learning round. Instead, the model leverages pre-known weights to focus solely on task-specific mapping, resulting in dramatic order-of-magnitude speedups.

Table~\ref{tab:ssl_timing} quantifies these gains across ImageNet subsets. We observe dramatic speedups, often exceeding \textbf{100$\times$}, particularly for uncertainty-based methods. For instance, on ImageNet-50, \textit{Margin} sampling sees a reduction in round time from 123 minutes to 0.68 minutes. Even at larger scales (ImageNet-200), the speedups remain substantial (e.g., $97\times$ for \textit{Entropy}).

\begin{table*}[t]
\centering
\caption{Average runtime per AL round (minutes) on ImageNet with and without MoCov3 embeddings ($B=1000$). SSL-based initialization consistently reduces operational overhead, with speedups highlighted in bold.}
\label{tab:ssl_timing}
\small
\makebox[\textwidth][c]{%
\resizebox{\textwidth}{!}{%
\begin{tabular}{l|cc|c|cc|c|cc|c}
\toprule
\textbf{Method} 
& \multicolumn{2}{c|} {\textbf{ImageNet-50}}  & \textbf{Speedup}
& \multicolumn{2}{c|}{\textbf{ImageNet-100}} & \textbf{Speedup}
& \multicolumn{2}{c|}{\textbf{ImageNet-200}} & \textbf{Speedup} \\
& no SSL & SSL & SSL / no SSL
& no SSL & SSL & SSL / no SSL
& no SSL & SSL & SSL / no SSL \\
\midrule
Random      & 92.82 & 0.57 & \textbf{162.84} & 184.27 & 1.56 & \textbf{118.12} & 184.27 & 3.17 & \textbf{58.13} \\
Margin      & 123.24 & 0.68 & \textbf{181.24} & 135.88 & 1.41 & \textbf{96.37} & 139.63 & 2.40 & \textbf{58.18} \\
Entropy     & 91.04 & 0.63 & \textbf{144.51} & 186.39 & 1.80 & \textbf{103.55} & 314.39 & 3.22 & \textbf{97.64} \\
Uncertainty & 93.01 & 0.76 & \textbf{122.38} & 179.47 & 1.52 & \textbf{118.07} & 193.61 & 3.07 & \textbf{63.07} \\
TypiClust   & 217.89 & 6.63 & \textbf{32.86} & 171.39 & 2.33 & \textbf{73.56} & 275.36 & 6.29 & \textbf{43.78} \\
ProbCover   & 87.37 & 1.06 & \textbf{82.42} & 135.92 & 1.74 & \textbf{78.11} & 207.17 & 4.51 & \textbf{45.94} \\
UHerding    & 144.26 & 1.20 & \textbf{120.22} & 170.61 & 2.64 & \textbf{64.62} & 260.99 & 9.22 & \textbf{28.31} \\
\bottomrule
\end{tabular}%
}}
\end{table*}

\paragraph{Algorithmic Scaling Analysis.}

Table~\ref{tab:timing_ratio} examines how acquisition costs scale from small to large budgets. Uncertainty-based methods demonstrate the most consistent scalability (ratios $\sim1.0$), as their primary cost-inference is independent of labeled pool size. In contrast, clustering-based methods like \textit{TypiClust} exhibit higher sensitivity to budget increases (e.g., a $6.79\times$ increase on CIFAR-10), likely due to the escalating complexity of density estimation and $K$-means clustering as the candidate pool evolves. These results emphasize that the transition to model-driven regimes (Phase III) is not only theoretically motivated by generalization bounds but also practically advantageous for large-scale system efficiency.

\begin{table*}[t]
\centering
\caption{Scaling analysis: Average runtime ratio between large ($B=1000$) and small budgets. Ratios $>2.0$ (highlighted) indicate methods where acquisition overhead scales poorly with label count.}
\label{tab:timing_ratio}
\small
\makebox[\textwidth][c]{%
\resizebox{\textwidth}{!}{%
\begin{tabular}{l|cc|c|cc|c|cc|c|cc|c|cc|c}
\toprule
\textbf{Method} 
& \multicolumn{2}{c|}{\textbf{CIFAR-10}} & \textbf{Ratio}
& \multicolumn{2}{c|}{\textbf{CIFAR-100}} & \textbf{Ratio}
& \multicolumn{2}{c|}{\textbf{ImageNet-50}} & \textbf{Ratio}
& \multicolumn{2}{c|}{\textbf{ImageNet-100}} & \textbf{Ratio}
& \multicolumn{2}{c|}{\textbf{ImageNet-200}} & \textbf{Ratio} \\
& $B_{10}$ & $B_{1000}$ & $R$ 
& $B_{100}$ & $B_{1000}$ & $R$ 
& $B_{50}$ & $B_{1000}$ & $R$ 
& $B_{100}$ & $B_{1000}$ & $R$ 
& $B_{200}$ & $B_{1000}$ & $R$ \\
\midrule
Random      & 1.49 & 2.50 & 1.68 & 1.44 & 2.56 & 1.78 & 57.64 & 92.83 & 1.61 & 70.10 & 184.27 & \textbf{2.63} & 170.32 & 184.27 & 1.08 \\
Margin      & 2.15 & 2.96 & 1.38 & 1.63 & 2.33 & 1.43 & 71.87 & 123.24 & 1.71 & 145.17 & 139.63 & 0.96 & 182.11 & 139.63 & 0.77 \\
Entropy     & 2.33 & 2.48 & 1.06 & 1.54 & 2.47 & 1.60 & 46.36 & 91.04 & 1.96 & 163.53 & 186.39 & 1.14 & 163.53 & 314.39 & \textbf{1.92} \\
Uncertainty & 2.33 & 2.42 & 1.04 & 1.59 & 2.26 & 1.42 & 47.07 & 93.01 & 1.98 & 92.30 & 179.47 & 1.94 & 284.97 & 193.61 & 0.68 \\
TypiClust   & 2.48 & 16.84 & \textbf{6.79} & 3.08 & 16.24 & \textbf{5.27} & 66.55 & 217.89 & \textbf{3.27} & 85.48 & 171.39 & \textbf{2.01} & 152.42 & 275.36 & 1.81 \\
ProbCover   & 2.38 & 2.92 & 1.23 & 1.33 & 2.32 & 1.74 & 43.24 & 87.37 & \textbf{2.02} & 145.77 & 135.92 & 0.93 & 199.73 & 207.17 & 1.04 \\
UHerding    & 1.47 & 4.07 & \textbf{2.77} & 2.01 & 2.99 & 1.49 & 56.55 & 144.26 & \textbf{2.55} & 94.15 & 170.61 & 1.81 & 226.19 & 260.99 & 1.15 \\
\bottomrule
\end{tabular}%
}}
\end{table*}

\end{document}